%% file: main.tex
\documentclass[english,runningheads]{llncs}

\input{macros_eccv}

\usepackage{subfig}
\usepackage{graphicx}
\usepackage{epsfig}
\usepackage{amsmath,amssymb} 
\usepackage{color}
\usepackage{pgfplots}
\usepackage[width=122mm,left=12mm,paperwidth=146mm,height=193mm,top=12mm,paperheight=217mm]{geometry}
\usepackage{scalefnt}
\usepackage{tikz}
\usepackage{pgfplots}
  \pgfplotsset{compat=newest}
  \usetikzlibrary{plotmarks}
  \usepackage{grffile}
\usetikzlibrary{spy,calc}

\newcommand{\figEnvelope}{
\begin{figure*}[h]
  \centering
  \subfloat[][\scriptsize Original dataterm]{
    \includegraphics[width=0.23\linewidth]{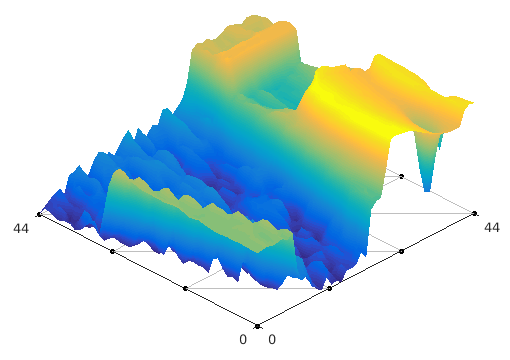}
    \label{fig:env_a}
  }
  \subfloat[][\scriptsize Without lifting]{
    \includegraphics[width=0.23\linewidth]{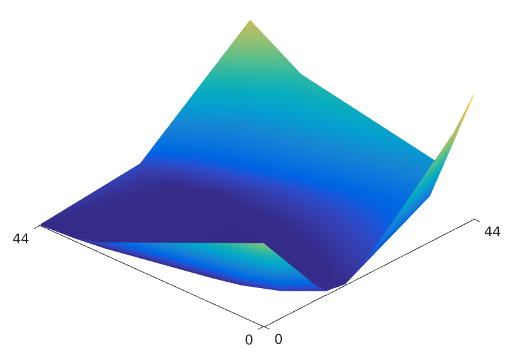}
    \label{fig:env_b}
  }
  \subfloat[][\scriptsize Classical lifting]{
    \includegraphics[width=0.23\linewidth]{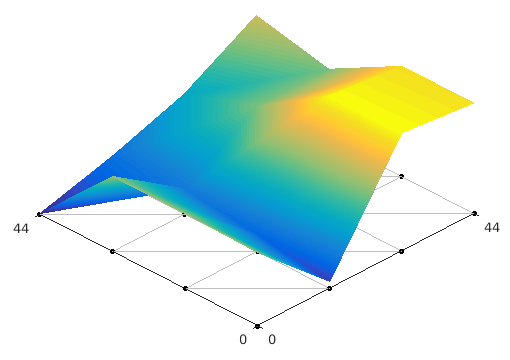}
    \label{fig:env_c}
  }
  \subfloat[][\scriptsize Proposed lifting]{
    \includegraphics[width=0.23\linewidth]{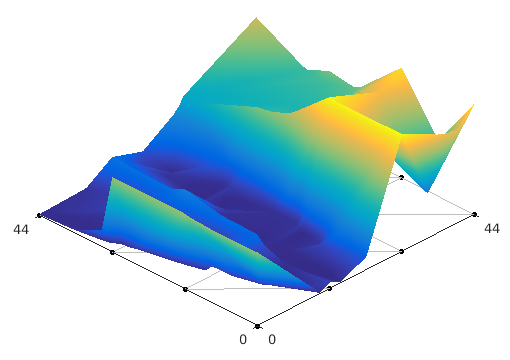}
    \label{fig:env_d}
  }
  \caption{
In (a) we show a nonconvex dataterm.
Convexification without lifting would result in the energy (b). Classical lifting methods~\cite{lellmann-et-al-iccv2013} (c), approximate the energy piecewise
    linearly between the labels, whereas the proposed method results in an approximation that is convex on each triangle (d).
   Therefore, we are able to capture the structure of the nonconvex energy
    much more accurately. 
}
  \label{fig:convex_envelope}
\end{figure*}
}

\newcommand{\figOneToOneCorresp}{
\begin{figure*}[t!]
  \centering
  \includegraphics{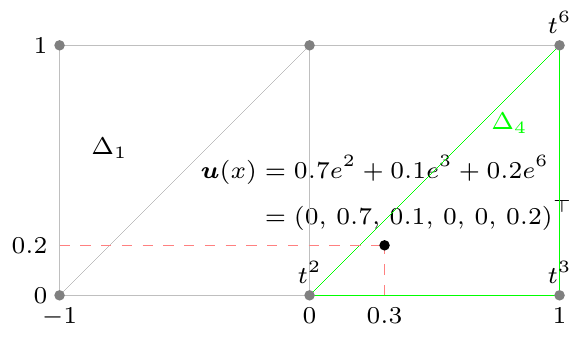}
  \caption{This figure illustrates our notation and the one-to-one
    correspondence between $u(x)=(0.3,0.2)^\top$ and the lifted
    $\ul(x)$ containing the barycentric coordinates
    $\alpha=(0.7,0.1,0.2)^\top$ of the sublabel $u(x)\in
    {\color{green} \Delta_4}=\mathrm{conv}\{t^2, t^3, t^6\}$. The triangulation $(\mathcal{V}, \mathcal{T})$ of $\Gamma=[-1;1]\times[0;1]$ is visualized via the gray lines, corresponding to the triangles and the gray dots, corresponding to the vertices $\mathcal{V} = \{(-1,0)^\top, (0,0)^\top, \dots, (1,1)^\top\}$, that we refer to as the labels.} 
  \label{fig:one_to_one}
\end{figure*}
}

\newcommand{\figBiconjGeom}{
\begin{figure*}[t!]
  \centering
  \captionsetup[subfloat]{labelformat=empty,justification=centering,singlelinecheck=false}
  \subfloat[][]{
    \includegraphics[width=0.35\linewidth]{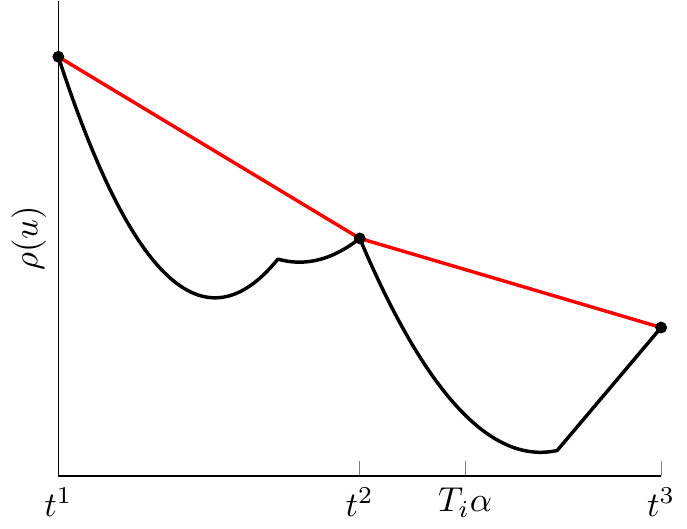}
    \label{fig:biconj_2d_standard}
  }
  \subfloat[][]{
    \includegraphics[width=0.35\linewidth]{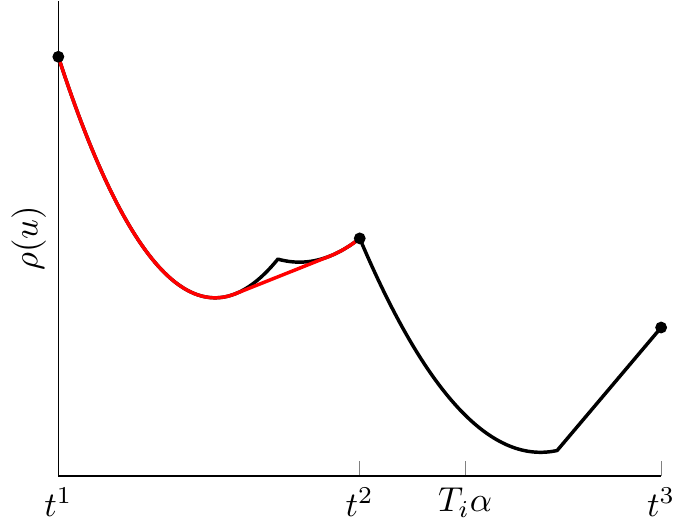}
    \label{fig:biconj_2d_standard}
  }\\
  \subfloat[][Standard lifting \cite{lellmann-et-al-iccv2013}]{
    \includegraphics[clip, trim=1.5cm 0.5cm 2.5cm 1.4cm,width=0.35\linewidth]{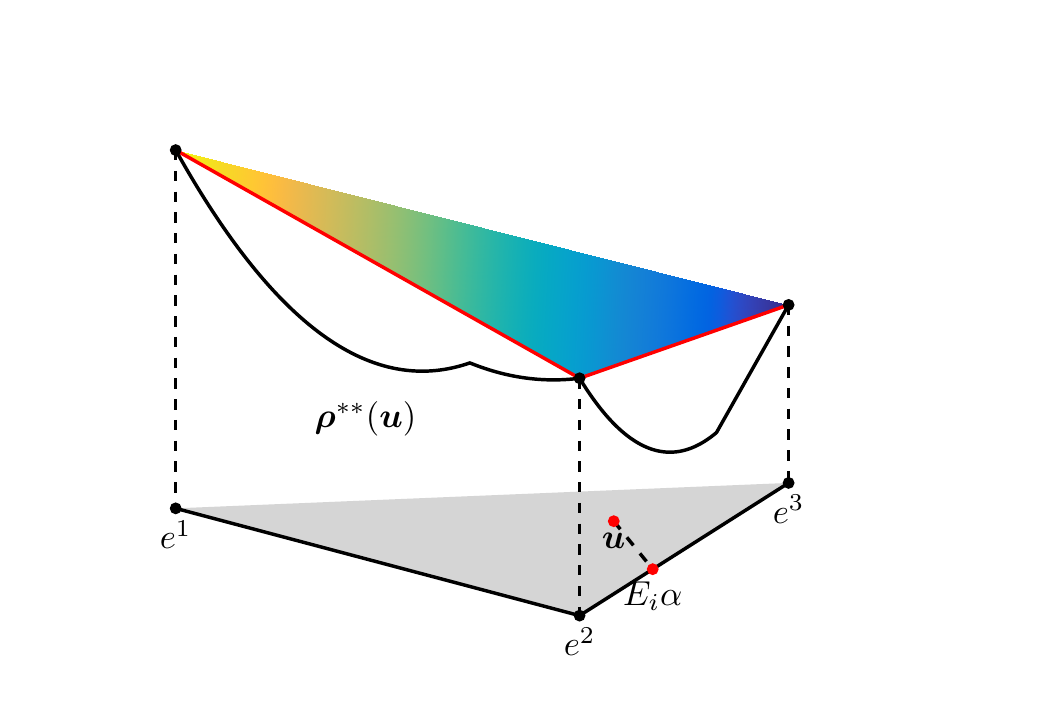}
    \label{fig:biconj_3d_standard}
  }
  \subfloat[][Proposed lifting]{
    \includegraphics[clip, trim=1.5cm 0.5cm 2.5cm 1.4cm,width=0.35\linewidth]{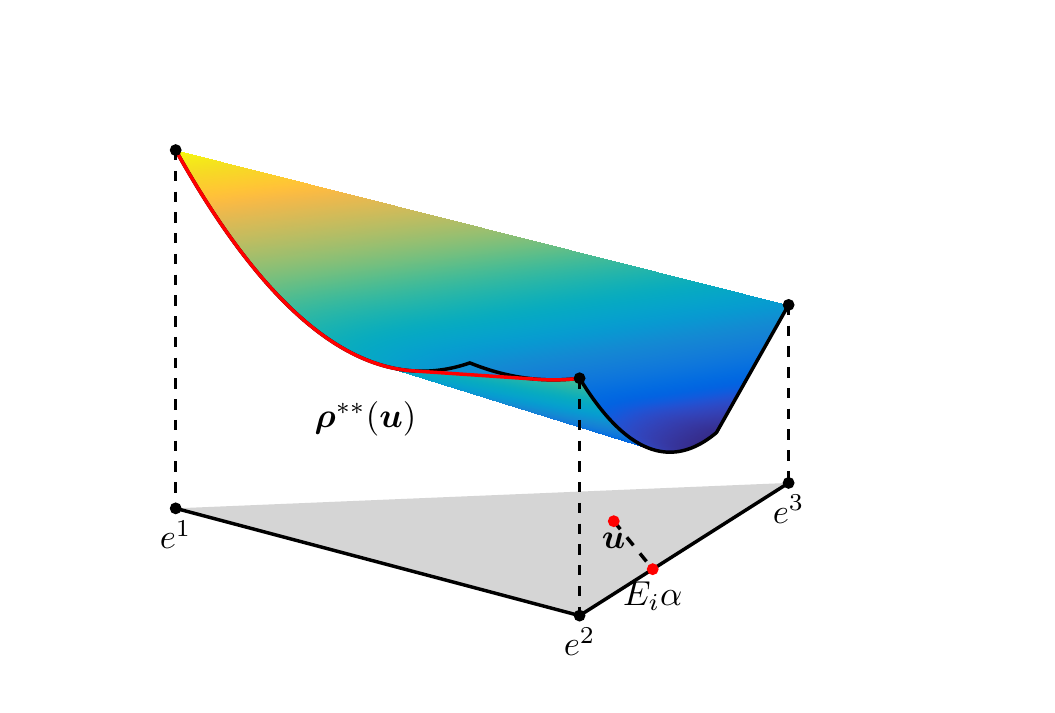}
    \label{fig:biconj_3d_standard}
  }
  \caption{Geometrical intuition for the proposed lifting and standard lifting~\cite{lellmann-et-al-iccv2013} for the special case of 1-dimensional range $\Gamma=[a,b]$ and $3$ labels $\{t^1,t^2,t^3\}$. 
  The standard lifting correponds to a linear interpolation of the
  original cost in between the locations $t^1,t^2,t^3$, which are
  associated with the vertices $e^1,e^2,e^3$ in the lifted energy
  (lower left). The proposed method extends the cost to the relaxed
  set in a more precise way: The original cost is preserved on the
  connecting lines between adjacent $e^i$ (black lines on the bottom right) up to concave parts (red graphs and lower surface on the right). This information, which may influence the exact location of the minimizer, is lost in the standard formulation. If the solution of the lifted formulation $\boldsymbol u$ is in the interior (gray area) an approximate solution to the original problem can still be obtained via Eq.~\eqref{eq:one2one_corr}.}
  \label{fig:biconj_3d}
\end{figure*}
}

\newcommand{\figSpirale}{
\begin{figure*}[t!]
  \centering
  \captionsetup[subfloat]{labelformat=empty,justification=centering,singlelinecheck=false}
  \subfloat[][Naive, 81 labels.]{
    \includegraphics[width=0.31\linewidth]{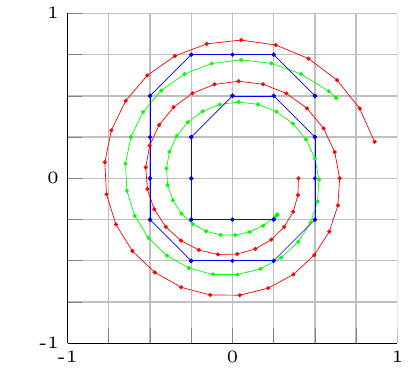}
    \label{fig:spirale_naive_l9}
  }
  \subfloat[][\cite{lellmann-et-al-iccv2013}, 81 labels.]{
    \includegraphics[width=0.31\linewidth]{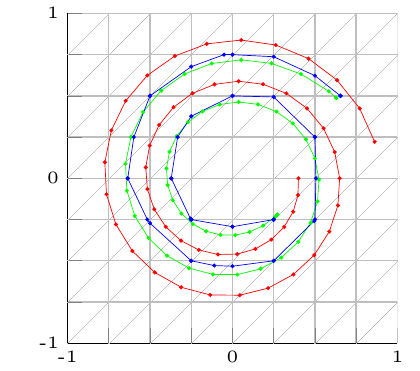}
    \label{fig:spirale_lellmann_l9}
  }
  \subfloat[][Ours, \textbf{4 labels}.]{
    \includegraphics[width=0.31\linewidth]{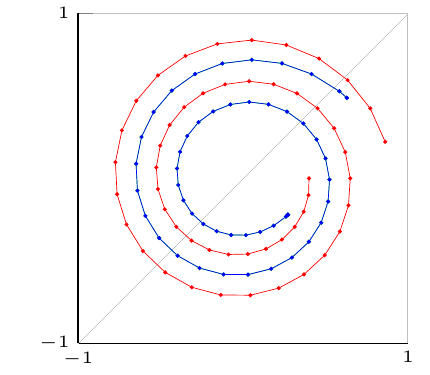}
    \label{fig:spirale_proposed_l2}
  }
  \caption{ROF denoising of a vector-valued signal $f : [0, 1] \to
    [-1, 1]^2$, discretized on 50 points (shown in red). We compare
    the proposed approach (right) with two alternative
    techniques introduced in  \cite{lellmann-et-al-iccv2013} (left and
    middle). The labels are visualized by the gray grid.  While the naive (standard) multilabel approach from
    \cite{lellmann-et-al-iccv2013} (left) provides solutions that are
    constrained to the chosen set of labels, the sublabel accurate
    regularizer from \cite{lellmann-et-al-iccv2013} (middle)
    does allow sublabel solutions, yet -- due to the dataterm bias --
    these still exhibit a strong preference for the grid points.  In
    contrast, the proposed approach does not exhibit any visible grid
    bias providing fully sublabel-accurate solutions:  With only $4$ labels, the computed solutions (shown in blue) coincide with the ``unlifted'' problem (green).
}

  \label{fig:spirale}
\end{figure*}
}

\newcommand{\figFlowBeanbags}{
\begin{figure*}[t!]
  \centering
  \subfloat[][\scriptsize Image 1 and 2]{
    \includegraphics[width=0.31\linewidth]{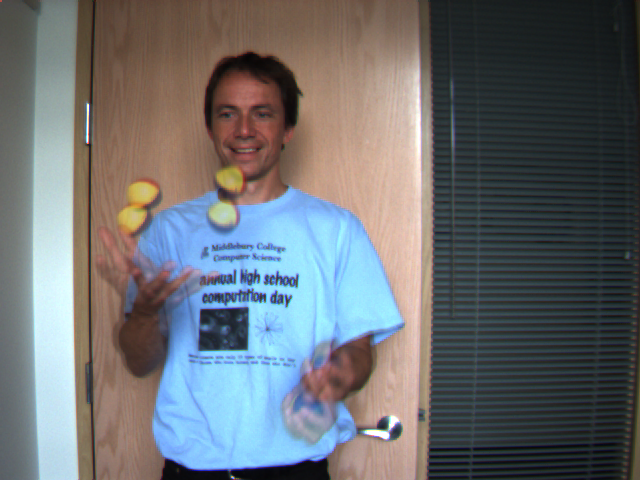}
    \label{fig:beanbags_input}
  }
  \subfloat[][\scriptsize Proposed, $\left\vert{\mathcal{V}}\right\vert=2 \times 2$]{
    \includegraphics[width=0.31\linewidth]{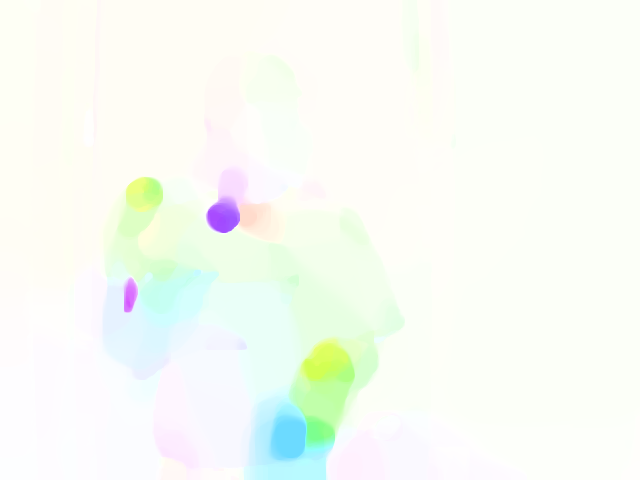}
    \label{fig:beanbags_flow_proposed}
  }
  \subfloat[][\scriptsize Baseline, $\left\vert{\mathcal{V}}\right\vert=7 \times 7$]{
    \includegraphics[width=0.31\linewidth]{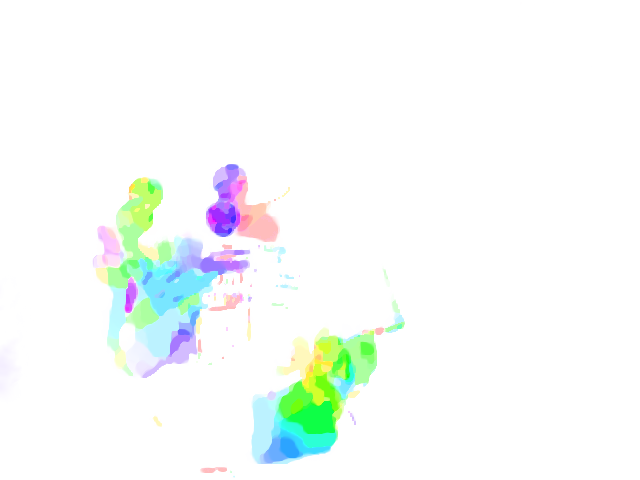}
    \label{fig:beanbags_flow_baseline}
  }
  \caption{Large displacement flow between two $640\times 480$ images (a) using a $81 \times 81$ search window. The result of our method with $4$ labels is shown in (b), the baseline \cite{lellmann-et-al-iccv2013} in (c). Our method can correctly identify the large motion.}
  \label{fig:flow_large_disp}
\end{figure*}
}

\newcommand{\figFlowGrove}{
\begin{figure*}[t!]
  \centering
  \captionsetup[subfloat]{labelformat=empty,justification=centering,singlelinecheck=false}
  \subfloat[][\scriptsize Image 1]{
    \includegraphics[width=0.185\linewidth]{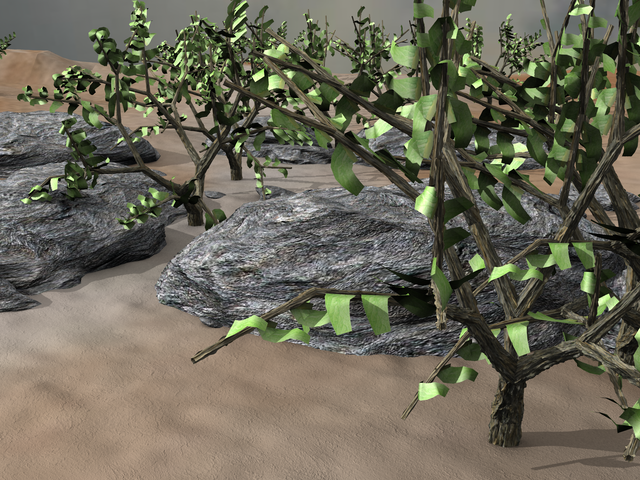}
  }
  \subfloat[][\scriptsize \cite{goldluecke-et-al-siam-2013},
  $\left\vert{\mathcal{V}}\right\vert=5\times 5$,\\ $0.67$ GB, $4$ min\\$\text{aep}=2.78$]{
    \includegraphics[width=0.185\linewidth]{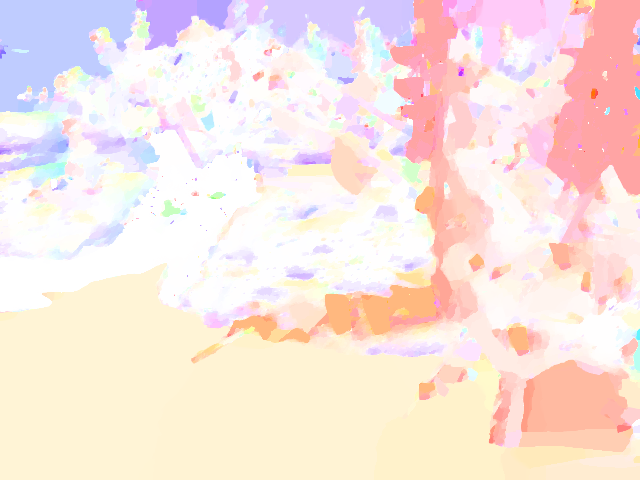}
  }
  \subfloat[][\scriptsize \cite{goldluecke-et-al-siam-2013},
  $\left\vert{\mathcal{V}}\right\vert=11\times11$,\\ $2.1$ GB, $12$
  min\\$\text{aep}=1.97$]{
    \includegraphics[width=0.185\linewidth]{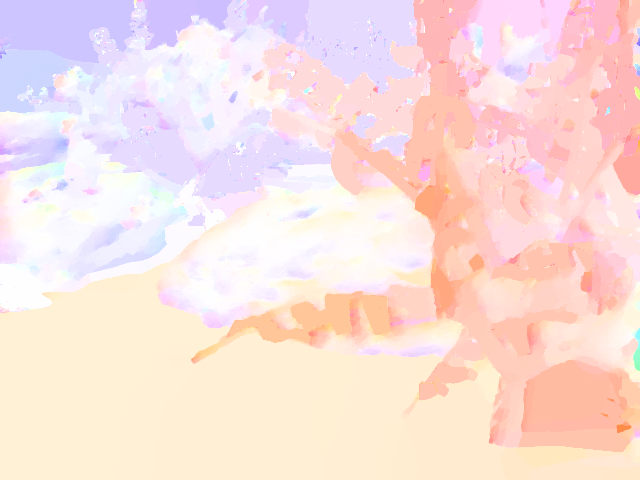}
  }
  \subfloat[][\scriptsize \cite{goldluecke-et-al-siam-2013},
  $\left\vert{\mathcal{V}}\right\vert=17\times17$,\\ $4.1$ GB, $25$
  min \\$\text{aep}=1.63$]{
    \includegraphics[width=0.185\linewidth]{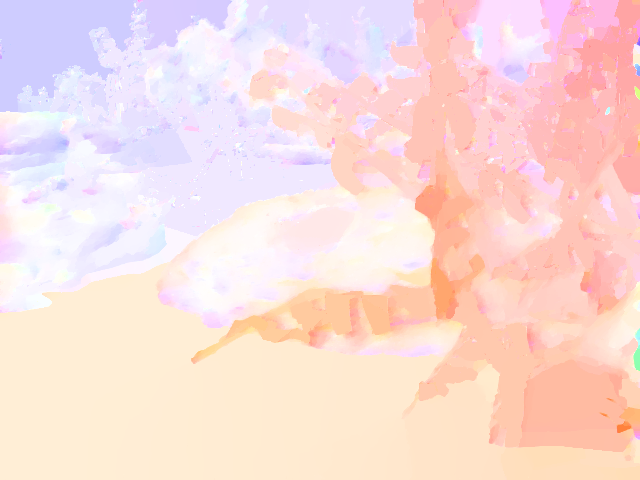}
  }
  \subfloat[][\scriptsize \cite{goldluecke-et-al-siam-2013},
  $\left\vert{\mathcal{V}}\right\vert=28\times28$,\\ $9.3$ GB, $60$
  min \\$\text{aep}=1.39$]{
    \includegraphics[width=0.185\linewidth]{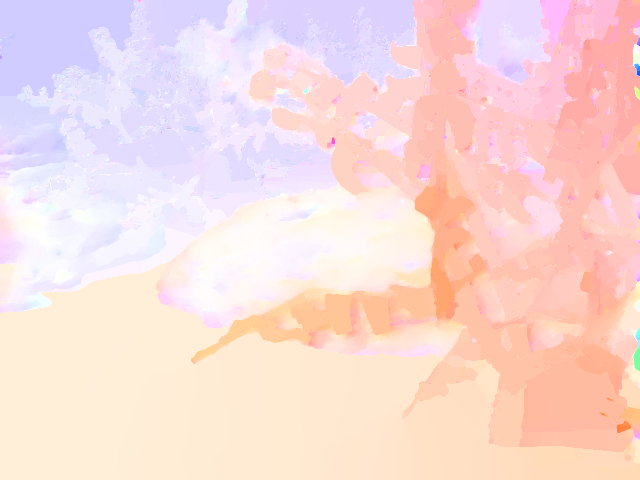}
  } \\

  \subfloat[][\scriptsize Image 2]{
    \includegraphics[width=0.185\linewidth]{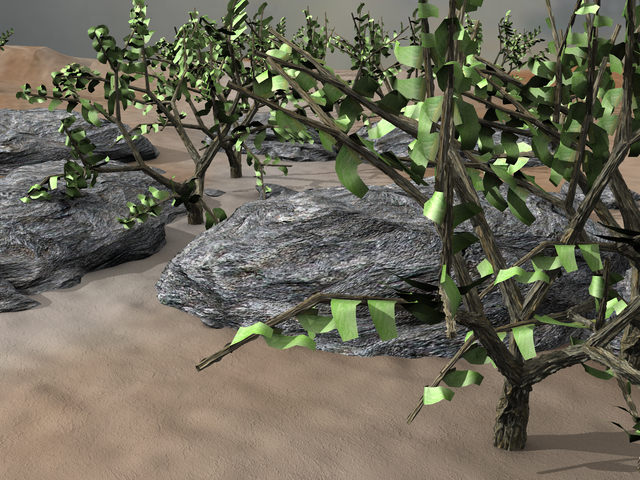}
  }
  \subfloat[][\scriptsize \cite{lellmann-et-al-iccv2013},
  $\left\vert{\mathcal{V}}\right\vert=3\times 3$,\\ $0.67$ GB, $0.35$ min\\$\text{aep}=5.44$]{
    \includegraphics[width=0.185\linewidth]{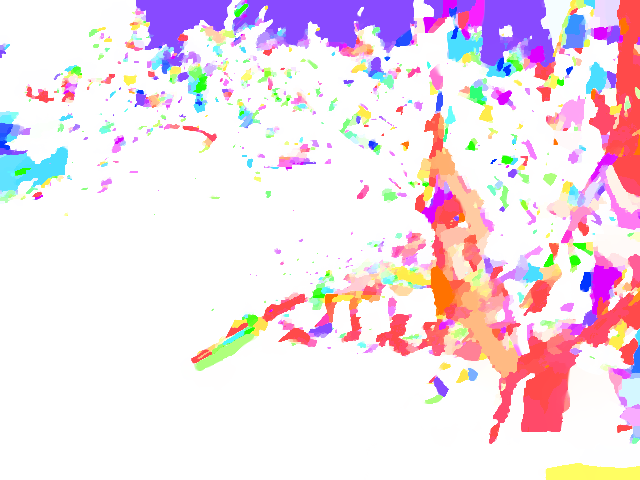}
  }
  \subfloat[][\scriptsize \cite{lellmann-et-al-iccv2013},
  $\left\vert{\mathcal{V}}\right\vert=5\times5$,\\ $2.4$ GB, $16$
  min\\$\text{aep}=4.22$]{
    \includegraphics[width=0.185\linewidth]{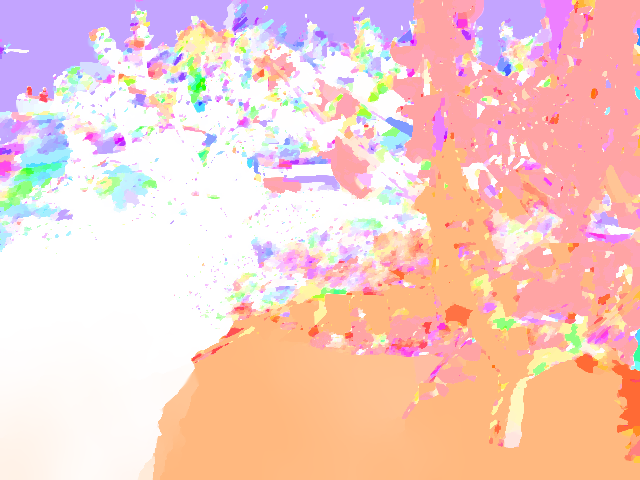}
  }
  \subfloat[][\scriptsize \cite{lellmann-et-al-iccv2013},
  $\left\vert{\mathcal{V}}\right\vert=7\times7$,\\ $5.2$ GB, $33$
  min \\$\text{aep}=2.65$]{
    \includegraphics[width=0.185\linewidth]{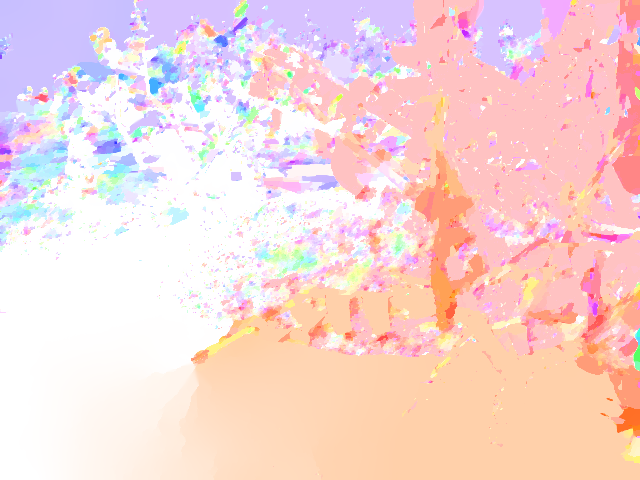}
  }
  \subfloat[][\scriptsize \cite{lellmann-et-al-iccv2013},
  $\left\vert{\mathcal{V}}\right\vert=9\times9$,\\ Out of memory.]{
    \includegraphics[width=0.185\linewidth]{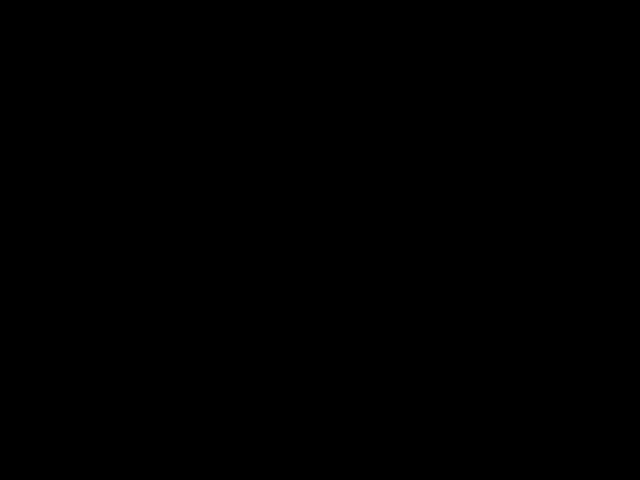}
  }
  \\
  \subfloat[][\scriptsize Ground truth]{
    \includegraphics[width=0.185\linewidth]{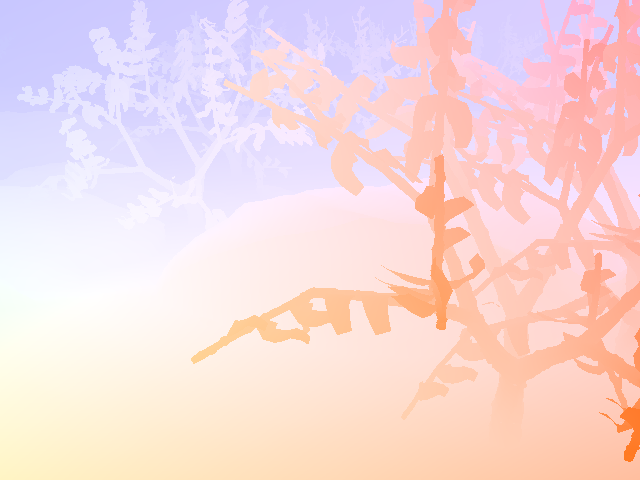}
  }
  \subfloat[][\scriptsize Ours,
  $\left\vert{\mathcal{V}}\right\vert=2\times 2$,\\ $0.63$ GB, $17$
  min \\$\text{aep}=1.28$]{
    \includegraphics[width=0.185\linewidth]{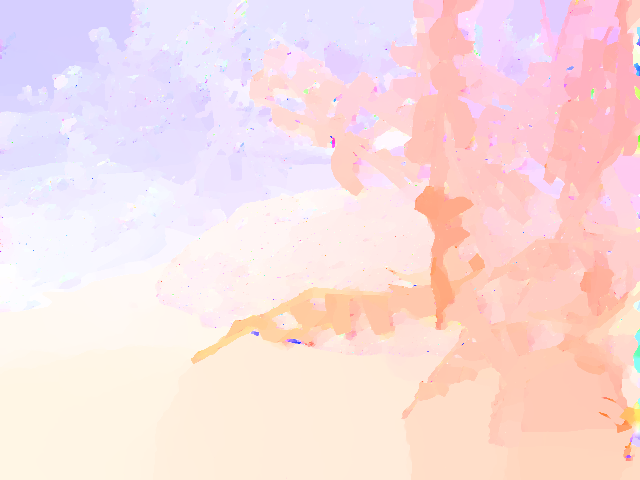}
  }
  \subfloat[][\scriptsize Ours,
  $\left\vert{\mathcal{V}}\right\vert=3\times 3$,\\ $1.9$ GB, $34$ min
  \\$\text{aep}=1.07$]{
    \includegraphics[width=0.185\linewidth]{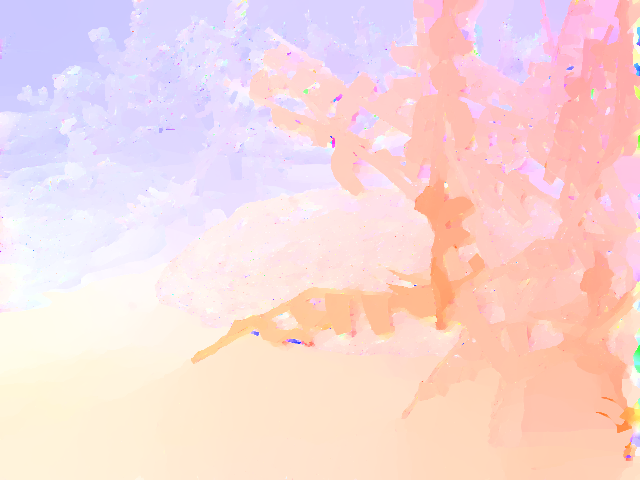}
  }
  \subfloat[][\scriptsize Ours,
  $\left\vert{\mathcal{V}}\right\vert=4\times 4$,\\ $4.1$ GB, $41$ min
  \\$\text{aep}=0.97$]{
    \includegraphics[width=0.185\linewidth]{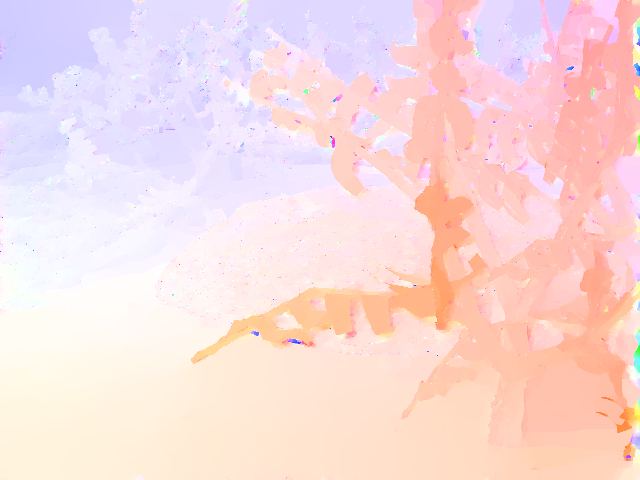}
  }
  \subfloat[][\scriptsize Ours,
  $\left\vert{\mathcal{V}}\right\vert=6\times 6$,\\ $10.1$ GB,
  $56$
  min \\$\text{aep}=0.9$]{
    \includegraphics[width=0.185\linewidth]{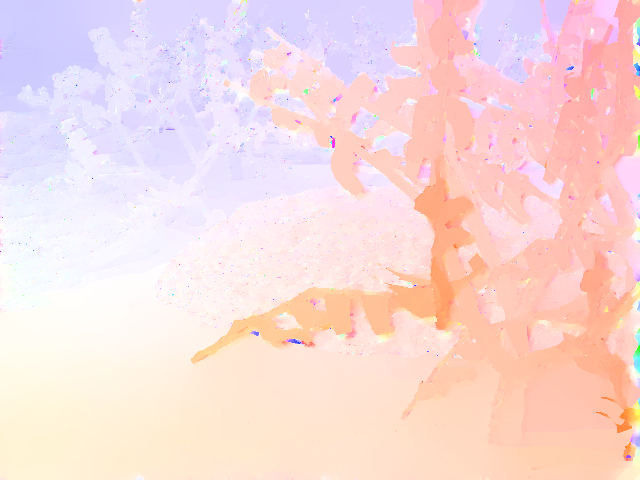}
  }

  \caption{We compute the optical flow using
    our method, the product space approach
    \cite{goldluecke-et-al-siam-2013} and the baseline method
    \cite{lellmann-et-al-iccv2013} for a varying amount of
    labels and compare the average endpoint error (aep). 
    The product space method clearly outperforms the baseline,
    but our approach finds the overall best result already with $2 \times 2$ labels.
    To achieve a similarly precise result as the product
    space method, we require $150$ times fewer labels, $10$ times less memory and $3$ times
    less time. For the same number of labels, the proposed approach requires more memory as it
    has to store a convex approximation of the energy instead of a linear one.
}
  \label{fig:flow_grove}
\end{figure*}
}

\newcommand{\figROF}{
\begin{figure*}[t!]
  \centering
  \captionsetup[subfloat]{labelformat=empty,justification=centering,singlelinecheck=false}
  \subfloat[][\scriptsize Input image]{
    \includegraphics[width=0.18\linewidth]{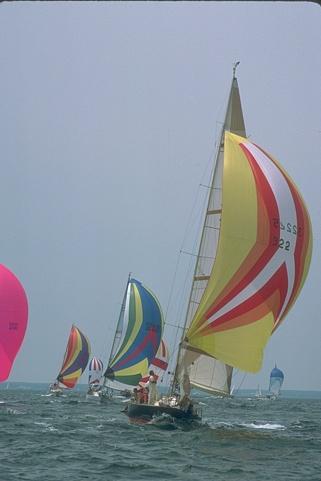}
  }
  \subfloat[][\scriptsize Unlifted Problem,\\ ~\\$E=992.50$]{
    \includegraphics[width=0.18\linewidth]{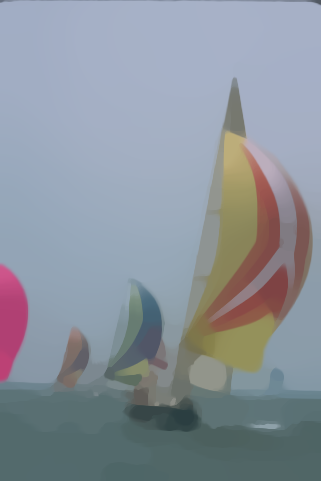}
  }
  \subfloat[][\scriptsize Ours, $\left\vert{\mathcal{T}}\right\vert=1$, \\$\left\vert{\mathcal{V}}\right\vert=4$, \\$E=992.51$]{
    \includegraphics[width=0.18\linewidth]{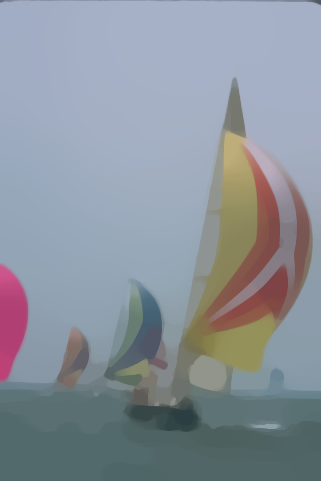}
  }
  \subfloat[][\scriptsize Ours, $\left\vert{\mathcal{T}}\right\vert=6$ \\$\left\vert{\mathcal{V}}\right\vert=2\times2\times2$ \\$E=993.52$]{
    \includegraphics[width=0.18\linewidth]{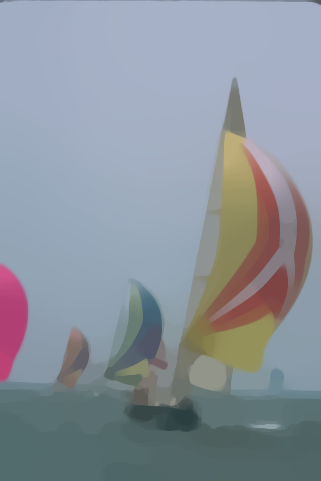}
  }
  \subfloat[][\scriptsize Baseline, \\ $\left\vert{\mathcal{V}}\right\vert=4\times4\times4$, \\ $E=2255.81$]{
    \includegraphics[width=0.18\linewidth]{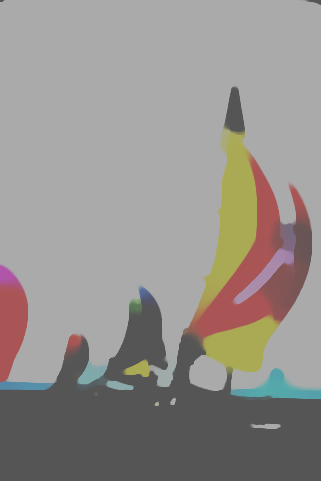}
  }\\
  \caption{Convex ROF with vectorial TV. Direct optimization and proposed method yield the same result. In contrast to the baseline method~\cite{lellmann-et-al-iccv2013} the proposed approach has no discretization artefacts and yields a lower energy. The regularization parameter is chosen as $\lambda=0.3$.}
  \label{fig:rof}
\end{figure*}
}

\newcommand{\figRobustROF}{
\begin{figure*}[t!]
  \centering
  \captionsetup[subfloat]{labelformat=empty,justification=centering,singlelinecheck=false}
  \subfloat[][\scriptsize Noisy input]{
    \includegraphics[width=0.18\linewidth]{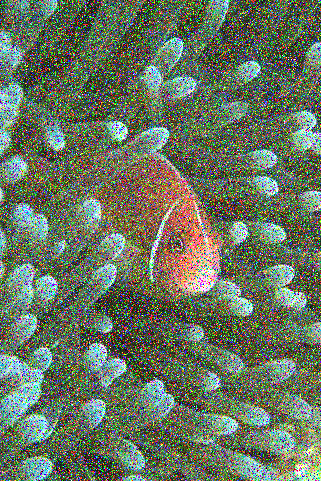}
  }
  \subfloat[][\scriptsize Ours, $\left\vert{\mathcal{T}}\right\vert=1$, $\left\vert{\mathcal{V}}\right\vert=4$, \\$E=2849.52$]{
    \includegraphics[width=0.18\linewidth]{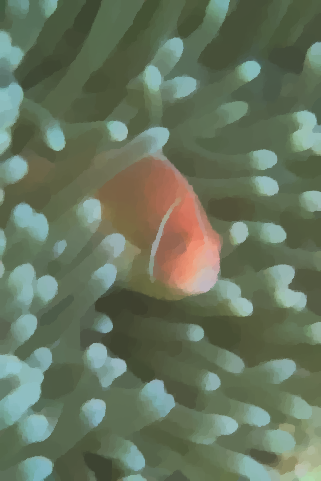}
  }
  \subfloat[][\scriptsize Ours, $\left\vert{\mathcal{T}}\right\vert=6$, $\left\vert{\mathcal{V}}\right\vert=2\times2\times2$, \\$E=2806.18$]{
    \includegraphics[width=0.18\linewidth]{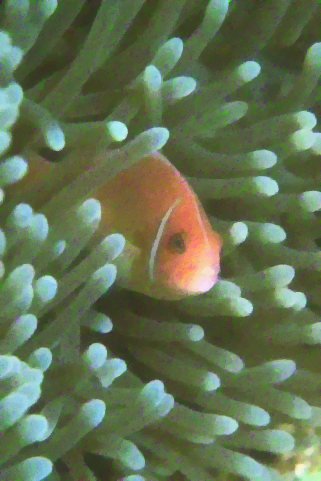}
  }
  \subfloat[][\scriptsize Ours, $\left\vert{\mathcal{T}}\right\vert=48$, $\left\vert{\mathcal{V}}\right\vert=3\times3\times3$, \\$E=2633.83$]{
    \includegraphics[width=0.18\linewidth]{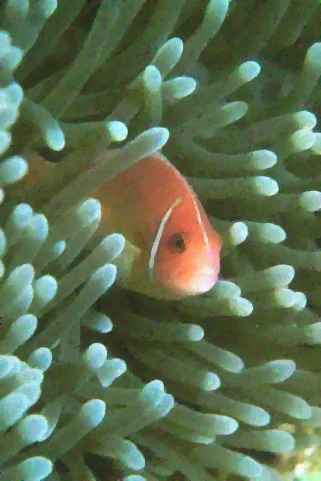}
  }
  \subfloat[][\scriptsize Baseline, $\left\vert{\mathcal{V}}\right\vert=4\times4\times4$, \\$E=3151.80$]{
    \includegraphics[width=0.18\linewidth]{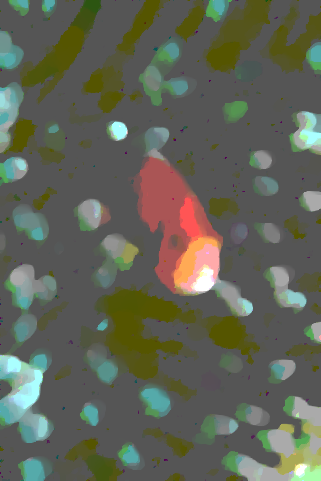}
  }\\
  \caption{ROF with a truncated quadratic dataterm ($\lambda=0.03$ and $\nu=0.025$). Compared to the baseline method~\cite{lellmann-et-al-iccv2013} the proposed approach yields much better results, already with a very small number of $4$ labels.}
  \label{fig:robust_rof}
\end{figure*}
}

\newcommand{\figProofIllustration}{
\begin{figure*}[t!]
  \centering
  \includegraphics{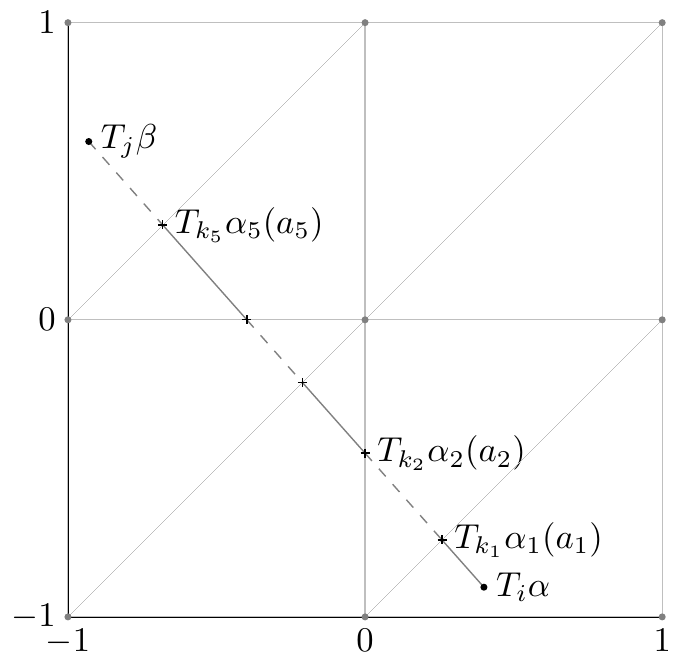}
  \caption{Figure illustrating the second direction of the proof of Proposition 4. The gray dots and lines visualize the triangulation $(\mathcal{V}, \mathcal{T})$. The line segment between $T_i \alpha$ and $T_j \beta$ is composed of shorter line segments which are fully contained in one of the triangles. On each of the triangles the inequality \eqref{eq:helper} holds, which allows to conclude that it holds for the whole line segment.}
  \label{fig:proof_illustration}
\end{figure*}
}

\newcommand{\figCartoonTextDecomp}{
\begin{figure*}[t!]
  \centering
  \captionsetup[subfloat]{labelformat=empty,justification=centering,singlelinecheck=false}
  \subfloat[][\scriptsize Input image]{
    \includegraphics[width=0.32\linewidth]{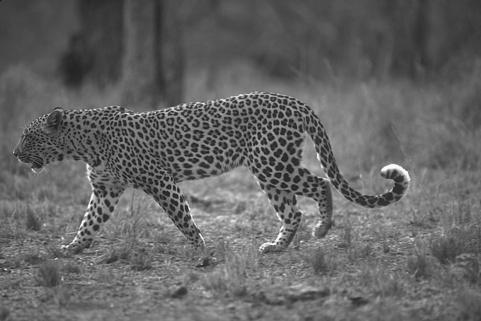}
  }
  \subfloat[][\scriptsize Mean $\mu$]{
    \includegraphics[width=0.32\linewidth]{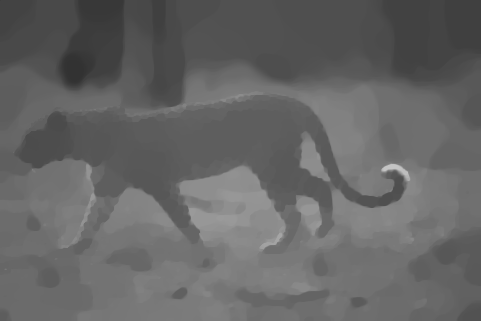}
  }
  \subfloat[][\scriptsize Variance $\sigma$]{
    \includegraphics[width=0.32\linewidth]{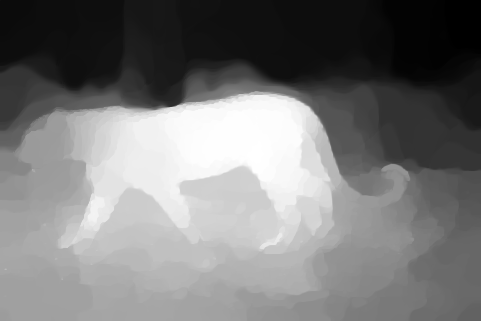}
  }\\
  \caption{Joint estimation of mean and variance. Our formulation can optimize difficult nonconvex joint optimization
    problems with continuous label spaces.}
  \label{fig:cartoon-text-decomp}
\end{figure*}
}

\begin{document}

\pagestyle{headings}
\mainmatter
\def\ECCV16SubNumber{189}  





\title{Sublabel-Accurate Convex Relaxation of Vectorial Multilabel Energies}
\titlerunning{Sublabel-Accurate Convex Relaxation of Vectorial Multilabel Energies}
\authorrunning{E. Laude \and T. M\"{o}llenhoff \and M. Moeller\and J. Lellmann \and D. Cremers}
\newcommand*\samethanks[1][\value{footnote}]{\footnotemark[#1]}
\author{Emanuel Laude\thanks{These authors contributed equally.}$^1$ \and Thomas M\"ollenhoff\samethanks$^1$  \and Michael Moeller$^1$ \and \\ Jan Lellmann$^2$ \and Daniel Cremers$^1$}
\institute{$^1$Technical University of Munich\thanks{This work was supported by the ERC Starting Grant ``Convex Vision''. } \qquad $^2$University of L\"ubeck}

\maketitle

\begin{abstract}
  Convex relaxations of multilabel problems have been
  demonstrated to produce provably optimal or near-optimal 
  solutions to a variety of computer vision problems.  Yet,
  they are of limited practical use as they require a fine
  discretization of the label space, entailing a huge demand in memory
  and runtime.  In this work, we propose the first sublabel accurate
  convex relaxation for vectorial multilabel problems.  Our key idea
  is to approximate the dataterm in a piecewise convex (rather than piecewise linear) manner.
  As a result we have a more faithful approximation of the original
  cost function that provides a meaningful interpretation for 
  fractional solutions of the relaxed convex problem.  

\end{abstract}

\figEnvelope

\section{Introduction}
\subsection{Nonconvex Vectorial Problems}
In this paper, we derive a sublabel-accurate convex relaxation for
vectorial optimization problems of the form 
\begin{equation}
  \underset{u : \Omega \to \Gamma} \min ~ \int_{\Omega} \rho\big(x, u(x)\big) \,\mathrm{d}x \,+\, \lambda \,TV(u),
  \label{eq:unlifted_problem}
\end{equation}
where $\Omega \subset \bbR^d$, $\Gamma \subset \bbR^n$ and $\rho :
\Omega \times \Gamma \to \bbR$ denotes a generally nonconvex pointwise
dataterm. As regularization we focus on the \emph{total
  variation} defined as: 
\begin{equation}
  TV(u) = \underset{q \in C_c^{\infty}(\Omega, \bbR^{n \times d}),
    \norm{q(x)}_{S^{\infty}} \leq 1} \sup ~ \int_{\Omega}
  \iprod{u}{\Div q} ~ \mathrm{d}x,
  \label{eq:total_variation}
\end{equation}
where $\norm{\cdot}_{S^{\infty}}$ is the Schatten-$\infty$ norm on
$\bbR^{n \times d}$, i.e., the largest singular value. 
For differentiable functions $u$ we can integrate
\eqref{eq:total_variation} by parts to find
\begin{equation}
  TV(u) = \int_{\Omega} \norm{\nabla u(x)}_{S^1} ~ \mathrm{d}x,
\end{equation}
where the dual norm $\norm{\cdot}_{S^1}$ penalizes the sum of the
singular values of the Jacobian, which encourages the individual components 
of $u$ to jump in the same direction. This type of regularization
is part of the framework of Sapiro and Ringach \cite{Sap96}.



\subsection{Related Work}
Due to its nonconvexity the optimization of
\eqref{eq:unlifted_problem} is challenging.  For the scalar
case ($n=1$), Ishikawa \cite{Ishikawa} proposed a pioneering technique
to obtain globally optimal solutions in a spatially discrete setting,
given by the minimum s-t-cut of a graph representing the space
$\Omega\times\Gamma$.  A continuous formulation was
introduced by Pock et al.~\cite{PockECCV} exhibiting several
advantages such as less grid bias and parallelizability.

In a series of papers \cite{PCBC-SIIMS,PCBC-ICCV09},
connections of the above approaches were made to the mathematical
theory of \emph{cartesian currents} \cite{GMS-CC} and the calibration
method for the Mumford-Shah functional \cite{Alberti-et-al-03},
leading to a generalization of the convex relaxation framework
\cite{PockECCV} to more general (in particular nonconvex)
regularizers.

In the following, researchers have strived to generalize the concept
of functional lifting and convex relaxation to the vectorial setting
($n > 1$).  If the dataterm and the regularizer are both separable in
the label dimension, one can simply apply the above convex
relaxation approach in a channel-wise manner to each component
separately.  But when either the dataterm or the regularizer
couple the label components, the situation becomes more complex \cite{goldluecke-et-al-siam-2013,strekalovskiy-et-al-siims14}.



The approach which is most closely related to our work, and which we
consider as a baseline method, is the one by Lellmann et
al.~\cite{lellmann-et-al-iccv2013}. They consider coupled dataterms
with coupled total variation regularization of the form \eqref{eq:total_variation}.

A drawback shared by all mentioned papers is that ultimately one
has to discretize the label space. While Lellmann et
al.~\cite{lellmann-et-al-iccv2013} propose a sublabel-accurate
regularizer, we show that their dataterm leads to solutions which
still have a strong bias towards the label grid.
For the scalar-valued setting, continuous label spaces have been
considered in the MRF community by Zach et
al.~\cite{Zach-Kohli-eccv12} and Fix et al.~\cite{Fix-eccv14}. The
paper \cite{Zach-aistats13} proposes a method for mixed continuous and
discrete vectorial label spaces, where everything is derived in the
spatially discrete MRF setting.  M\"{o}llenhoff et
al.~\cite{moellenhoff-laude-cvpr-2016} recently proposed a novel
formulation of the scalar-valued case which retains fully continuous
label spaces even after discretization. The contribution of this work is
to extend \cite{moellenhoff-laude-cvpr-2016} to vectorial label
spaces, thereby complementing \cite{lellmann-et-al-iccv2013} with a
sublabel-accurate dataterm.


\subsection{Contribution}
In this work we propose the first sublabel-accurate convex formulation
of vectorial labeling problems.  It generalizes the formulation for
scalar-valued labeling problems~\cite{moellenhoff-laude-cvpr-2016} and thus includes
important applications such as optical flow estimation or color image denoising.  We show that our method, derived in a spatially continuous
setting, has a variety of interesting theoretical properties as well
as practical advantages over the existing labeling approaches:
\begin{itemize}
\item We generalize existing functional lifting
  approaches (see Sec.~\ref{sec:convex_dataterm}).
\item We show that our method is the best convex
  under-approximation (in a local sense), see Prop.~\ref{prop:conv_env_data} and Prop.~\ref{prop:conv_env_reg}.
\item Due to its sublabel-accuracy our method requires only a small
  amount of labels to produce good results which leads to a
  drastic reduction in memory.  We believe that this is a
  vital step towards the real-time capability of lifting and convex
  relaxation methods.  Moreover, our method eliminates the label bias,
  that previous lifting methods suffer from, even for many labels.
\item In Sec.~\ref{sec:regularizer} we propose a regularizer that couples the
  different label components by enforcing a joint jump normal. This is in contrast to~\cite{goldluecke-et-al-siam-2013}, where the components are regularized
  separately.
\item For convex dataterms, our method is equivalent to the unlifted problem -- see Prop.~\ref{prop:equivalency_unlifted}.
  Therefore, it allows a seamless transition between direct
  optimization and convex relaxation approaches.
\end{itemize}

\subsection{Notation}
We write $\iprod{x}{y} = \sum_i x_i y_i$ for the standard inner
product on $\bbR^n$ or the Frobenius product if $x,y$ are
matrices. Similarly $\norm{\cdot}$ without any subscript denotes 
the usual Euclidean norm, respectively the Frobenius norm for matrices. 
 
We denote the convex conjugate
of a function $f : \bbR^n \to \bbR \cup \{ \infty \}$ by
$f^*(y) = \sup_{x \in \bbR^n} ~ \iprod{y}{x} - f(x)$. It is an
important tool for devising convex relaxations,
as the biconjugate $f^{**}$ is the largest lower-semicontinuous
(lsc.) convex function below $f$. 
For the indicator function of a set $C$ we write $\delta_C$, i.e.,
$\delta_C(x) = 0$ if $x \in C$ and $\infty$ otherwise. $\Delta_n^U
\subset \R^n$ stands for the unit $n$-simplex. 
\section{Convex Formulation}
\subsection{Lifted Representation}
\label{sec:lifting}
Motivated by Fig.~\ref{fig:convex_envelope}, we construct an equivalent
representation of \eqref{eq:unlifted_problem} in a higher dimensional
space, before taking the convex envelope.

Let $\Gamma \subset \bbR^n$ be a compact and convex set. We partition
$\Gamma$ into a set $\mathcal{T}$ of $n$-simplices $\Delta_i$ so that
$\Gamma$ is a disjoint union of $\Delta_i$ up to a set of measure
zero. Let $t^{i_j}$ be the $j$-th vertex of $\Delta_i$ and denote by $\mathcal{V}= \{t^1,\ldots,t^{\left\vert{\mathcal{V}}\right\vert}\}$ the union of all vertices, referred to as labels, with $1\leq i \leq \left\vert{\mathcal{T}}\right\vert$, $1 \leq j \leq n+1$ and $1 \leq i_j \leq \left\vert{\mathcal{V}}\right\vert$. For $u : \Omega \to \Gamma$, we refer to $u(x)$ as a \emph{sublabel}. Any sublabel can be written as a convex combination of the vertices of a simplex $\Delta_i$ with $1\leq i \leq \left\vert{\mathcal{T}}\right\vert$ for appropriate barycentric coordinates $\alpha \in \Delta_n^U$:
\begin{equation}
u(x)=T_i \alpha:= \sum_{j=1}^{n+1}\alpha_j t^{i_j}, ~ T_i:=(t^{i_1},\ t^{i_2},\ \dots, \ t^{i_{n+1}}) \in \bbR^{n \times n+1}.
\end{equation}
By encoding the vertices $t^k \in \mathcal{V}$ using a one-of-$\left\vert{\mathcal{V}}\right\vert$ representation $e^k$ we can identify any $u(x) \in \Gamma$ with a sparse vector $\ul(x)$ containing at least $\left\vert{\mathcal{V}}\right\vert-n$ many zeros and vice versa:
\begin{equation}
\begin{split}
\ul(x)&=E_i \alpha  :=\sum_{j=1}^{n+1}\alpha_j e^{i_j}, ~ E_i:=(e^{i_1}, \ e^{i_2}, \dots, \ e^{i_{n+1}}) \in \bbR^{\left\vert{\mathcal{V}}\right\vert \times n+1}, \\
u(x)&= \sum_{k=1}^{\left\vert{\mathcal{V}}\right\vert} t^k \ul_k(x), ~ \alpha \in \Delta_n^U, ~ 1\leq i \leq \left\vert{\mathcal{T}}\right\vert.
\label{eq:one2one_corr}
\end{split}
\end{equation}
The entries of the vector $e^{i_j}$ are zero except for the $(i_j)$-th entry, which is equal to one.
We refer to $\ul : \Omega \to \bbR^{\left\vert{\mathcal{V}}\right\vert}$ as the \emph{lifted} representation of $u$.
This one-to-one-correspondence between $u(x) = T_i \alpha$ and $\ul(x) = E_i \alpha$ is shown in Fig.~\ref{fig:one_to_one}. Note that both, $\alpha$ and $i$ depend on $x$. However, for notational convenience we drop the dependence on $x$ whenever we consider a fixed point $x \in \Omega$.
\figOneToOneCorresp

\subsection{Convexifying the Dataterm}
\label{sec:convex_dataterm}
Let for now the weight of the regularizer in~\eqref{eq:unlifted_problem} be zero. Then, at each point $x \in \Omega$ we minimize a generally nonconvex energy over a compact set $\Gamma \subset \bbR^n$:
\begin{equation}
  \underset{u \in \Gamma} \min ~ \rho(u).
  \label{eq:dataterm_min_simple}
\end{equation}
\figBiconjGeom
We set up the lifted energy so that it attains finite values if and only if the argument $\ul$ is a sparse representation $\ul = E_i \alpha$ of a sublabel $u \in \Gamma$:
\begin{equation}
  \dat(\ul) = \min_{1\leq i \leq \left\vert{\mathcal{T}}\right\vert} ~ \dat_i(\ul), \qquad   \dat_i(\ul) = 
  \begin{cases}
    \rho(T_i \alpha), \qquad &\text{if } ~ \ul = E_i \alpha, ~ \alpha \in \Delta_n^U,\\
    \infty, & \text{otherwise.}
  \end{cases}
  \label{eq:dataterm_lifted}
\end{equation}
Problems~\eqref{eq:dataterm_min_simple} and~\eqref{eq:dataterm_lifted} are equivalent due to the one-to-one correspondence of $u=T_i \alpha$ and $\ul = E_i \alpha$. However, energy \eqref{eq:dataterm_lifted} is finite on a nonconvex set only. In order to make optimization tractable, we minimize its convex envelope.
\begin{prop}
  The convex envelope of \eqref{eq:dataterm_lifted} is given as:
  \begin{equation} \label{eq:dataterm_sublabel_biconj}
  \begin{split}
    \boldsymbol \rho^{**}(\boldsymbol u) &= \sup_{\boldsymbol v \in \R^{\left\vert{\mathcal{V}}\right\vert}} \la \boldsymbol u, \boldsymbol v\ra - \max_{1 \leq i \leq \left\vert{\mathcal{T}}\right\vert} ~ \dat_i^*(\boldsymbol v), \\
    \dat_i^*(\boldsymbol v)&= \la E_i b_i, \boldsymbol v \ra + \rho_i^*(A_i^\top E_i^\top \boldsymbol v), \quad \rho_i := \rho + \delta_{\Delta_i}.
   \end{split}
  \end{equation}
$b_i$ and $A_i$ are given as $b_i := M_i^{n+1}$, $A_i := \left(M_i^1, ~ M_i^2,~ \dots,~ M_i^n \right)$, 
where $M_i^j$ are the columns of the matrix $M_i:=(T_i^\top, \mathbf{1})^{-\top} \in \R^{{n+1} \times {n+1}}$.
\label{prop:conv_env_data}
\end{prop}
\begin{proof}
  Follows from a calculation starting at the
  definition of $\dat^{**}$. See Appendix~\ref{appendix:a} for a 
  detailed derivation.
\end{proof}
The geometric intuition of this construction is depicted in Fig.~\ref{fig:biconj_3d}. Note that if one prescribes the value of $\dat_i$ in \eqref{eq:dataterm_lifted} only on the \emph{vertices} of the unit simplices $\Delta_n^U$, i.e., $\dat(\ul)=\rho(t^k)$ if $\ul=e^k$ and $+\infty$ otherwise, one obtains the linear biconjugate $\dat^{\ast\ast}(\ul) = \langle \ul, \boldsymbol{s} \rangle,\; \boldsymbol{s} = (\rho(t^i),\ldots,\rho(t^L))$ on the feasible set. This coincides with the standard relaxation of the dataterm used in \cite{PCBC-SIIMS,Lellmann2011c,Chambolle-et-al-siims12,lellmann-et-al-iccv2013}. In that sense, our approach can be seen as a relaxing the dataterm in a more precise way, by incorporating the true value of $\rho$ not only on the finite set of labels $\mathcal{V}$, but also everywhere in between, i.e., on every \emph{sublabel}.

\subsection{Lifting the Vectorial Total Variation} \label{sec:regularizer}
We define the lifted vectorial total variation as
\begin{equation}
  \boldsymbol{TV}(\ul) = \int_{\Omega} \boldsymbol\Psi(D \ul),
  \label{eq:lifted_TV}
\end{equation}
where $D\ul$ denotes the distributional derivative of $\ul$ and $\boldsymbol \Psi$ is positively one-homogeneous, i.e., $\boldsymbol{\Psi}(c \ul) = c\, \boldsymbol{\Psi}(\ul), c \geqs0$. For such functions, the meaning of \eqref{eq:lifted_TV} can be made fully precise using the polar decomposition of the Radon measure $D \ul$ \cite[Cor.~1.29, Thm.~2.38]{AFP}. However, in the following we restrict ourselves to an intuitive motivation for the derivation of $\boldsymbol{\Psi}$ for smooth functions.

Our goal is to find $\boldsymbol{\Psi}$ so that $\boldsymbol{TV}(\ul) = \TV(u)$ whenever $\ul:\Om\to \bbR^{\left\vert{\mathcal{V}}\right\vert}$ corresponds to some $u:\Om\to\Gamma$, in the sense that $\ul(x) = E_i \alpha$ whenever $u(x) = T_i \alpha$. In order for the equality to hold, it must in particular hold for all $u$ that are classically differentiable, i.e., $D u = \nabla u$, and whose Jacobian $\nabla u(x)$ is of rank 1, i.e., $\nabla u(x) = (T_i \alpha - T_j \beta) \otimes \nu(x)$ for some $\nu(x) \in \Rd$. This rank 1 constraint enforces the different components of $u$ to have the same jump normal, which is desirable in many applications. In that case, we observe
\begin{equation}
TV(u) = \int_\Omega \| T_{i} \alpha - T_{j} \beta \| \cdot \|\nu(x)\| \, \mathrm{d} x.
\end{equation}
For the corresponding lifted representation $\ul$, we have $\nabla \ul(x) = (E_i \alpha - E_j \beta) \otimes \nu(x)$. Therefore it is natural to require $\boldsymbol{\Psi}(\nabla \ul(x)) = \boldsymbol{\Psi}\left( (E_i \alpha - E_j \beta) \otimes \nu(x) \right) := \|T_i \alpha -T_j \beta\| \cdot \|\nu(x)\|$ in order to achieve the goal $\boldsymbol{TV}(\ul) = \TV(u)$.
Motivated by these observations, we define
\begin{equation}
\boldsymbol\Psi(\boldsymbol p):=\begin{cases} 
\| T_i \alpha - T_j \beta \| \cdot \|\nu\| & \text{if } ~ \boldsymbol p=(E_i \alpha - E_j \beta)\otimes \nu, \\
\infty & \text{otherwise},
\end{cases}
\end{equation}
where $\alpha, \beta \in \Delta_{n+1}^U$, $\nu \in \R^d$ and $1 \leq i,j \leq \left\vert{\mathcal{T}}\right\vert$.
Since the convex envelope of \eqref{eq:lifted_TV} is intractable, we derive a ``locally'' tight convex underapproximation:
\begin{equation}
  \boldsymbol{R}(\ul) = \sup_{\ql : \Omega \to \bbR^{d \times \left\vert{\mathcal{V}}\right\vert}} \int_{\Omega} \iprod{\ul}{\Div \ql} - \boldsymbol \Psi^*(\ql) \ \mathrm{d}x.
\end{equation}
\begin{prop}
The convex conjugate of $\boldsymbol\Psi$ is
\begin{equation}
\boldsymbol\Psi^*(\boldsymbol q)=\delta_\mathcal{K}(\boldsymbol q)
\end{equation}
with convex set
\begin{equation}
\label{eq:setK}
\begin{split}
\mathcal{K}&=\bigcap_{1\leq i,j \leq \left\vert{\mathcal{T}}\right\vert} \left\{ \boldsymbol q \in \R^{d \times \left\vert{\mathcal{V}}\right\vert} \bigm| \|Q_i \alpha - Q_j \beta\| \leq \| T_i \alpha- T_j \beta\|, \ \alpha, \beta \in \Delta_{n+1}^U \right\},
\end{split}
\end{equation}
and $Q_i = (\ql^{i_1},\ \ql^{i_2}, \ \dots, \ \ql^{i_{n+1}}) \in \bbR^{d \times n+1}$. $\boldsymbol q^j \in \mathbb{R}^d$ are the columns of $\boldsymbol q$.
\label{prop:conv_env_reg}
\end{prop}
\begin{proof}
  Follows from a calculation starting at the definition of the convex conjugate $\boldsymbol{\Psi}^*$. See Appendix~\ref{appendix:a}.
\end{proof}
Interestingly, although in its original formulation~\eqref{eq:setK} the set $\mathcal{K}$ has infinitely many constraints, one can equivalently represent $\mathcal{K}$ by finitely many. 

\begin{prop}
The set $\mathcal{K}$ in equation \eqref{eq:setK} is the same as 
\begin{equation}
\mathcal{K}=\left\{\ql \in \R^{d \times \left\vert{\mathcal{V}}\right\vert} \mid \left\|D_{\ql}^i \right\|_{S^\infty}\leq 1, \ 1\leq i\leq \left\vert{\mathcal{T}}\right\vert \right\}, ~ D_{\ql}^i = Q_iD\,(T_iD)^{-1}, 
\label{eq:DqMatrix}
\end{equation}
where the matrices $Q_iD \in \mathbb{R}^{d \times n}$ and $T_iD \in \mathbb{R}^{n \times n}$ are given as
$$
Q_iD := \left( \ql^{i_1} - \ql^{i_{n+1}}, ~ \dots, ~\ql^{i_n} - \ql^{i_{n+1}}\right), ~ T_iD:=\left(t^{i_1} - t^{i_{n+1}}, ~ \dots, ~t^{i_n} - t^{i_{n+1}}\right).
$$
\end{prop}
\begin{proof}
Similar to the analysis in \cite{lellmann-et-al-iccv2013}, equation \eqref{eq:setK} basically states the Lipschitz continuity of a piecewise linear function defined by the matrices $\ql \in \R^{d \times \left\vert{\mathcal{V}}\right\vert}$. Therefore, one can expect that the Lipschitz constraint is equivalent to a bound on the derivative. For the complete proof, see Appendix~\ref{appendix:a}.
\end{proof}

\subsection{Lifting the Overall Optimization Problem}
Combining dataterm and regularizer, the overall optimization problem is given 
\begin{equation}
  \min_{\ul : \Omega \to \bbR^{\left\vert{\mathcal{V}}\right\vert}} \sup_{\ql : \Omega \to \mathcal{K}} \int_{\Omega} \dat^{**}(\ul) + \iprod{\ul}{\Div \ql} \ \mathrm{d}x.
  \label{eq:opti_prob}
\end{equation}
A highly desirable property is that, opposed to any other vectorial lifting approach from the literature, our method with just one simplex 
applied to a convex problem yields the same solution as the unlifted problem.
\begin{prop} \label{prop:equivalency_unlifted}
  If the triangulation contains only 1 simplex, $\mathcal{T} = \{ \Delta \}$, i.e., $\left\vert{\mathcal{V}}\right\vert = n + 1$, then the proposed optimization problem \eqref{eq:opti_prob} is equivalent to 
  \begin{equation}
    \underset{u : \Omega \to \Delta} \min ~ \int_{\Omega} (\rho + \delta_\Delta)^{**}(x, u(x)) \ \mathrm{d}x + \lambda TV(u),
    \label{eq:unlifted_problem_convex}
  \end{equation}
  which is \eqref{eq:unlifted_problem} with a globally convexified dataterm on $\Delta$.
\end{prop}
\begin{proof}
For $u = t^{n+1} + TD \tilde u$ the substitution
$\ul = \left ( \tilde u_1, \hdots, \tilde u_n, 1 - \sum_{j=1}^n \tilde u_j
\right)$ into $\dat^{**}$ and $\boldsymbol R$ yields the result. For
a complete proof, see Appendix~\ref{appendix:a}.
\end{proof}

\section{Numerical Optimization}
\subsection{Discretization}
For now assume that $\Omega\subset \bbR^d$ is a $d$-dimensional Cartesian grid and let $\Div$ denote a finite-difference divergence operator with $\Div \boldsymbol q : \Omega \to \mathbb{R}^{\left\vert{\mathcal{V}}\right\vert}$. Then the relaxed energy minimization problem becomes
\begin{equation} \label{eq:discrete_saddle_point}
\min_{\ul : \Omega \to \mathbb{R}^{\left\vert{\mathcal{V}}\right\vert}}\max_{\boldsymbol q : \Omega \to \mathcal{K}} ~ \sum_{x \in \Omega} \dat^{**}(x, \ul(x)) + \langle \Div \boldsymbol q, \ul\rangle.
\end{equation}
In order to get rid of the pointwise maximum over $\dat_i^*(\boldsymbol v)$ in Eq.~\eqref{eq:dataterm_sublabel_biconj}, we introduce additional variables $w(x) \in \mathbb{R}$ and additional constraints $(\vl(x), w(x)) \in \mathcal{C}$, $x \in \Omega$ so that $w(x)$ attains the value of the pointwise maximum:
\begin{equation}\label{eq:discrete_saddle_point_final}
\min_{\ul : \Omega \to \mathbb{R}^{\left\vert{\mathcal{V}}\right\vert}} \max_{\substack{(\vl, w):\Omega \to \mathcal{C}\\[0.5mm] \boldsymbol q : \Omega \to \mathcal{K}}} ~ \sum_{x \in \Omega} \langle \ul(x), \vl(x) \rangle - w(x) + \langle \Div\boldsymbol q, \ul\rangle,
\end{equation}
where the set $\mathcal{C}$ is given as
\begin{equation}
 \mathcal{C} = \bigcap_{1\leq i \leq \left\vert{\mathcal{T}}\right\vert} \mathcal{C}_i, \quad \mathcal{C}_i:=\left \{ (x, y) \in \bbR^{\left\vert{\mathcal{V}}\right\vert+1} \mid \dat_i^*(x) \leq y \right \}.
    \label{eq:constraint_C}
\end{equation}
For numerical optimization we use a GPU-based
implementation\footnote{\url{https://github.com/tum-vision/sublabel_relax}}
of a first-order primal-dual method \cite{PCBC-ICCV09}. The algorithm
requires the orthogonal projections of the dual variables onto the
sets $\mathcal{C}$ respectively $\mathcal{K}$ in every
iteration. However, the projection onto an epigraph of dimension
$\left\vert{\mathcal{V}}\right\vert+1$ is difficult for large values
of $\left\vert{\mathcal{V}}\right\vert$. We rewrite the
constraints $(\vl(x), w(x)) \in \mathcal{C}_i$, $1\leq i \leq
\left\vert{\mathcal{T}}\right\vert$, $x\in\Omega$ as
$(n+1)$-dimensional epigraph constraints introducing variables
$r^i(x)\in \bbR^n$, $s_i(x)\in \bbR$:
\begin{equation} \label{eq:epi-small}
  \rho_i^*\left(r^i(x)\right) \leq s_i(x), \quad r^i(x)=A_i^\top E_i^\top \, \boldsymbol v(x), \quad s_i(x) = w(x) - \la E_i b_i, \boldsymbol v(x) \ra.
\end{equation}
These equality constraints can be implemented using Lagrange
multipliers. For the
projection onto the set $\mathcal{K}$ we use an approach similar to
\cite[Figure 7]{goldluecke-siims-2012}.

\subsection{Epigraphical Projections}\label{sec:epi-proj}
 Computing the Euclidean projection onto the epigraph of $\rho_i^*$ is a central part of the numerical implementation of the presented method. However, for $n > 1$ this is nontrivial. Therefore we provide a detailed explanation of the projection methods used for different classes of $\rho_i$. We will consider quadratic, truncated quadratic and piecewise linear $\rho$.
\paragraph{Quadratic case:}
Let $\rho$ be of the form $\rho(u)=\frac{a}{2} \, u^\top u + b^\top u
+ c$. A direct projection onto the epigraph of $\rho_i^*=(\rho +
\delta_{\Delta_i})^*$ for $n > 1$ is difficult. However, the epigraph
can be decomposed into separate epigraphs for which it is easier to
project onto: For proper, convex, lsc. functions $f,g$ the epigraph of
$(f+g)^*$ is the Minkowski sum of the epigraphs of $f^*$ and $g^*$ (cf.~\cite[Exercise~1.28,~Theorem~11.23a]{Rockafellar-Variational-Analysis}). This means that it suffices to compute the projections onto the epigraphs of a quadratic function $f^*=\rho^*$ and a convex, piecewise linear function $g^*(v)=\max_{1 \leq j \leq n+1} \iprod{t^{i_j}}{v}$ by rewriting constraint~\eqref{eq:epi-small} as 
\begin{equation}
\rho^*(r_{f})\leq s_{f},~ {\delta_{\Delta_i}}^*(c_{g}) \leq d_{g} ~ \text{ s.t. } (r,s)=(r_{f},s_{f})+(c_{g},d_{g}).
\end{equation}
For the projection onto the epigraph of a $n$-dimensional quadratic
function we use the method described in~\cite[Appendix
B.2]{strekalovskiy-et-al-siims14}. The projection onto a piecewise
linear function is described in the last paragraph of this section.
\paragraph{Truncated quadratic case:} 
Let $\rho$ be of the form $\rho(u)= \min \,\{\,\nu, ~ \frac{a}{2}\,
u^\top u + b^\top u + c\,\}$ as it is the case for the nonconvex
robust ROF with a truncated quadratic dataterm in
Sec.~\ref{sec:trunc-quad-rof}. Again, a direct projection onto the
epigraph of $\rho_i^*$ is difficult. However, a decomposition of the
epigraph into simpler epigraphs is possible as the epigraph of
$\min\{f, g\}^*$ is the intersection of the epigraphs of $f^*$ and
$g^*$. Hence, one can separately project onto the epigraphs
of $(\nu + \delta_{\Delta_i})^*$ and $(\frac{a}{2}\,
u^\top u + b^\top u + c + \delta_{\Delta_i})^*$. Both of these
projections can be handled using the methods from the other paragraphs.
\paragraph{Piecewise linear case:}
In case $\rho$ is piecewise linear on each $\Delta_i$, i.e., $\rho$
attains finite values at a discrete set of sampled sublabels
$\mathcal{V}_i \subset \Delta_i$ 
and interpolates linearly between them, we have that 
\begin{equation}
  (\rho+ \delta_{\Delta_i})^*(v)=\max_{\tau \in \mathcal{V}_i} ~ \iprod{\tau}{v} - \rho(\tau).
\end{equation}
Again this is a convex, piecewise linear function. For the projection
onto the epigraph of such a function, a quadratic program of the form
\begin{equation}
\min_{(x,y) \in \bbR^{n+1}} ~ \frac{1}{2}\|x-c\|^2 + \frac{1}{2}\|y-d\|^2 ~\text{ s.t. } \iprod{\tau}{x} - \rho(\tau) \leq y, \forall \tau \in \mathcal{V}_i
\end{equation}
needs to be solved.
We implemented the primal active-set method described in~\cite[Algorithm 16.3]{numerical-optimization-wright}, and found it solves the program in a few (usually $2-10$) iterations for a moderate number of constraints.

\section{Experiments}
\subsection{Vectorial ROF Denoising}
In order to validate experimentally, that our model is exact for convex dataterms, we evaluate it on the Rudin-Osher-Fatemi~\cite{Rudin-Osher-Fatemi-92} (ROF) model with vectorial TV~\eqref{eq:total_variation}. In our model this corresponds to defining $\rho(x, u(x))= \frac{1}{2}\|u(x) - I(x)\|^2$.
As expected based on Prop.~\ref{prop:equivalency_unlifted} the energy of the solution of the unlifted problem is equal to the energy of the projected solution of our method for $|\mathcal{V}|=4$ up to machine precision, as can be seen in Fig.~\ref{fig:spirale} and Fig.~\ref{fig:rof}. We point out, that the sole purpose of this experiment is a proof of concept as our method introduces an overhead and convex problems can be solved via direct optimization. It can be seen in Fig.~\ref{fig:spirale} and Fig.~\ref{fig:rof}, that the baseline method~\cite{lellmann-et-al-iccv2013} has a strong label bias.
\figSpirale
\figROF
\subsection{Denoising with Truncated Quadratic Dataterm} \label{sec:trunc-quad-rof}
For images degraded with both, Gaussian and salt-and-pepper noise we define the dataterm as
$\rho(x, u(x))= \min\,\left\{\frac{1}{2}\|u(x) - I(x)\|^2, \nu \right\}$.
We solve the problem using the epigraph decomposition described in the
second paragraph of Sec.~\ref{sec:epi-proj}. It can be seen, that
increasing the number of labels $\left\vert{\mathcal{V}}\right\vert$
leads to lower energies and at the same time to a reduced effect of
the TV. This occurs as we always compute a piecewise convex
underapproximation of the original nonconvex dataterm, that gets
tighter with a growing number of labels. The baseline method~\cite{lellmann-et-al-iccv2013} again produces strong discretization artefacts even for a large number of labels $\left\vert{\mathcal{V}}\right\vert=4\times4\times4=64$.
\figRobustROF
\figFlowGrove
\figFlowBeanbags 
\subsection{Optical Flow}
We compute the optical flow $v:\Omega \to \bbR^2$ between two input images $I_1, I_2$. The label space $\Gamma = [-d, d]^2$ is chosen
according to the estimated maximum displacement $d \in \bbR$ between
the images. 
The dataterm is 
  $\rho(x, v(x)) = \norm{I_2(x) - I_1(x + v(x))}$,
and $\lambda(x)$ is based on the norm of the image gradient $\nabla I_1(x)$.

In Fig.~\ref{fig:flow_grove} we compare the proposed method to the 
product space approach \cite{goldluecke-et-al-siam-2013}. Note
that we implemented the product space dataterm using Lagrange 
multipliers, also referred to as the \emph{global} approach in
\cite{goldluecke-et-al-siam-2013}. While this increases the memory 
consumption, it comes with lower computation time and guaranteed 
convergence.
For our method, we sample the label space $\Gamma = [-15, 15]^2$ on 
$150 \times 150$ sublabels and subsequently convexify the energy on 
each triangle using the quickhull algorithm~\cite{barber1996quickhull}. 
For the product space approach we sample the label space at
equidistant labels, from $5 \times 5$ to $27 \times 27$. 
As the regularizer from the product space approach is different from
the proposed one, we chose $\mu$ differently for each method. For
the proposed method, we set $\mu = 0.5$ and for the product space
and baseline approach $\mu = 3$.
We can see in Fig.~\ref{fig:flow_grove}, our method outperforms the
product space approach w.r.t. the average end-point error. 
Our method outperforms previous lifting approaches: In Fig.~\ref{fig:flow_large_disp} we compare our method on large displacement
optical flow to the baseline \cite{lellmann-et-al-iccv2013}. To obtain
competitive results on the Middlebury benchmark, 
one would need to engineer a better dataterm.

\section{Conclusions}

We proposed the first sublabel-accurate convex relaxation of vectorial
multilabel problems.  To this end, we approximate the generally
nonconvex dataterm in a piecewise convex manner as opposed to the
piecewise linear approximation done in the traditional functional
lifting approaches.  This assures a more faithful approximation of the
original cost function and provides a meaningful interpretation for
the non-integral solutions of the relaxed convex problem.  In
experimental validations on large-displacement optical flow estimation
and color image denoising, we show that the computed solutions have
superior quality to the traditional convex relaxation methods while
requiring substantially less memory and runtime.
\newpage
\begin{appendix}
\section{Theory} \label{appendix:a}
\begin{proof}[Proof of Proposition 1]
By definition the biconjugate of $\dat$ is given as
\begin{equation}
\begin{split}
\dat^{**}(\ul) &= \sup_{\vl \in \R^{\left\vert{\mathcal{V}}\right\vert}} \la \ul, \vl\ra - \left(\min_{1 \leq i \leq \left\vert{\mathcal{T}}\right\vert} \dat_i(\vl)\right)^* \\
&= \sup_{\vl \in \R^{\left\vert{\mathcal{V}}\right\vert}} \la \ul, \vl\ra - \max_{1 \leq i \leq \left\vert{\mathcal{T}}\right\vert} \dat_i^*(\vl).
\end{split}
\end{equation}
We proceed computing the conjugate of $\dat_i$:
\begin{equation}
\begin{split}
\boldsymbol\rho_i^*(\vl) &= \sup_{\ul \in  \R^{\left\vert{\mathcal{V}}\right\vert}} \la \ul, \vl\ra - \dat_i(\ul)\\
&=\sup_{\boldsymbol \alpha \in \Delta_{n+1}^U} \la E_i \alpha, \vl \ra  - \rho\left(T_i \alpha \right),
\end{split}
\end{equation}
We introduce the substitution $r:=T_i \alpha \in \Delta_i$ and obtain
\begin{equation}
\alpha = K_i^{-1} \begin{pmatrix}r\\ 1 \end{pmatrix}, \quad K_i:=\begin{pmatrix} T_i \\ \mathbf{1}^\top \end{pmatrix} \in \R^{{n+1} \times {n+1}},
\end{equation}
since $K_i$ is invertible for $(\mathcal{V}, \mathcal{T})$ being a non-degenerate triangulation and $\sum_{j=1}^{n+1} \alpha_j = 1$.
With this we can further rewrite the conjugate as
\begin{equation}
\begin{split}
\hdots &=\sup_{\boldsymbol r \in \Delta_i} \la A_i r+b_i, E_i^\top \vl\ra - \rho(r) \\
&= \la E_i b_i, \vl \ra + \sup_{r \in \R^n} \la r, A_i^\top E_i^\top  \vl\ra - \rho(r) - \delta_{\Delta_i}(r) \\
&= \la E_i b_i,\vl \ra + \rho_i^*(A_i^\top E_i^\top  \vl).
\end{split}
\end{equation}
\end{proof}

\begin{proof}[Proof of Proposition 2]
Define $\boldsymbol \Psi_{i,j}$ as
\begin{equation}
\boldsymbol\Psi_{i,j}(\pl):=\begin{cases} 
\| T_i \alpha - T_j \beta \| \cdot \|\nu\| & \text{if } ~ \pl=(E_i \alpha - E_j \beta)\nu^\top, ~ \alpha, \beta \in \Delta_{n+1}^U, ~ \nu \in \R^d, \\
\infty & \text{otherwise}.
\end{cases}
\end{equation}
Then, $\boldsymbol \Psi$ can be rewritten as a pointwise minimum over the individual $\boldsymbol \Psi_{i,j}$
\begin{equation}
\boldsymbol \Psi( \pl ) = \min_{1 \leq i,j \leq \left\vert{\mathcal{T}}\right\vert} \boldsymbol \Psi_{i,j}( \pl ).
\end{equation}
We begin computing the conjugate of $\boldsymbol \Psi_{i,j}$
\begin{equation}
\begin{split}
\boldsymbol\Psi_{i,j}^*(\ql)&= \sup_{\pl \in \R^{d \times \left\vert{\mathcal{V}}\right\vert}} \la \pl, \ql\ra - \boldsymbol \Psi_{i,j}(\pl) \\
&=\sup_{\alpha,\beta \in \Delta_{n+1}^U}\sup_{\nu \in \R^d} \la Q_i \alpha - Q_j \beta, \nu \ra - \| T_i \alpha - T_j \beta \| \cdot \|\nu\| \\
&= \sup_{\alpha,\beta \in \Delta_{n+1}^U} \left(\| T_i \alpha - T_j \beta \| \cdot \|\cdot\|\right)^*(Q_i \alpha - Q_j \beta) \\
&=\delta_{\mathcal{K}_{i,j}}(\ql),
\end{split}
\end{equation}
with the set $K_{i,j}$ being defined as
\begin{equation}
\mathcal{K}_{i,j} :=\left\{ \ql \in \R^{d \times \left\vert{\mathcal{V}}\right\vert} \bigm| \|Q_i \alpha - Q_j \beta\| \leq \| T_i \alpha- T_j \beta\|, \ \alpha, \beta \in \Delta_{n+1}^U \right\}.
\end{equation}
Since the maximum over indicator functions of sets is equal to the indicator function of the intersection of the sets we obtain for $\boldsymbol \Psi^*$
\begin{equation}
\begin{split}
\boldsymbol \Psi^{*}(\ql)&= \max_{1 \leq i,j \leq \left\vert{\mathcal{T}}\right\vert} \boldsymbol \Psi^*_{i,j}(\ql) \\
&= \delta_\mathcal{K}(\ql).
\end{split}
\end{equation}
\end{proof}

\begin{proof}[Proof of Proposition 3]
\figProofIllustration
Let $\ql \in \R^{d \times \left\vert{\mathcal{V}}\right\vert}$ s.t. $\|Q_i \alpha - Q_j \beta\| \leq \| T_i \alpha- T_j \beta\|$ for all $\alpha, \beta \in \Delta_{n+1}^U$ and $1\leq i,j \leq \left\vert{\mathcal{T}}\right\vert$. For any $1\leq i\leq \left\vert{\mathcal{T}}\right\vert$ define 
\begin{equation}
\begin{split}
f_i: \mathbb{R}^n &\rightarrow \mathbb{R}^n, \\
(\alpha_1, ...,\alpha_n) &\mapsto \sum_{l=1}^n \alpha_l t^{i_l} + (1- \sum_{l=1}^n \alpha_l) t^{i_{n+1}} = T_i\alpha,
\end{split}
\end{equation}
and analogously
\begin{equation}
\begin{split}
g_i: \mathbb{R}^n &\rightarrow \mathbb{R}^{ \left\vert{\mathcal{V}}\right\vert} \\
(\alpha_1, ...,\alpha_n) &\mapsto \sum_{l=1}^n \alpha_l \ql^{i_l} + (1- \sum_{l=1}^n \alpha_l) \ql^{i_{n+1}} = Q_i \alpha.
\end{split}
\end{equation}
Let us choose an $\alpha \in \mathbb{R}^n$ such that $\alpha_i>0$, $\sum_l \alpha_l <1$. Then $\|Q_i \alpha - Q_j \beta\| \leq \| T_i \alpha- T_j \beta\|$ for all $\alpha, \beta \in \Delta_{n+1}^U$ and $1\leq i,j \leq \left\vert{\mathcal{T}}\right\vert$ implies that 
\begin{equation}
\|g_i(\alpha) - g_i(\alpha - h)\| \leq \|f_i(\alpha) - f_i(\alpha - h)\|,
\end{equation}
holds for all vectors $h$ with sufficiently small entries. Inserting the definitions of $g_i$ and $f_i$ we find that
\begin{equation}
\| Q_i D h \| \leq \|T_i D h\|
\end{equation}
holds for all $h$ with sufficiently small entries. For a non-degenerate triangle, $T_i D$ is invertible and a simple substitution yields that \begin{equation}
\| Q_iD(T_i D)^{-1} \tilde h \|_2 \leq \| \tilde h\|,
\end{equation}
holds for all $\tilde h$ with sufficiently small entries. This means that the operator norm of $D_{\ql}^i$ induced by the $\ell^2$ norm, i.e. the $S^\infty$ norm, is bounded by one. 

Let us now show the other direction. For $\ql \in \R^{d \times \left\vert{\mathcal{V}}\right\vert}$ s.t. $\left\|D_{\ql}^i \right\|_{S^\infty}\leq 1, \ 1\leq i\leq \left\vert{\mathcal{T}}\right\vert$, note that inverting the above computation immediately yields that 
\begin{equation}
\label{eq:helper}
\|Q_k\alpha - Q_k\beta\| \leq \|T_k \alpha - T_k \beta\|
\end{equation}
holds for all $1\leq k \leq  \left\vert{\mathcal{T}}\right\vert$, $\alpha, \beta \in \Delta_{n+1}^U$. Our goal is to show that having this inequality on each simplex is sufficient to extend it to arbitrary pairs of simplices. The overall idea of this part of the proof is illustrated in Fig.~\ref{fig:proof_illustration}. 

Let $1\leq i,j\leq  \left\vert{\mathcal{T}}\right\vert$ and $\alpha, \beta \in \mathbb{R}^n$ with $\alpha_l, \beta_l \geq 0$, $\sum_l \alpha_l \leq \sum_l \beta_l \leq 1$ be given. Consider the line segment
\begin{equation}
\begin{split}
c(\gamma): [0,1] &\rightarrow \mathbb{R}^d \\
\gamma &\mapsto \gamma \,T_j \beta + (1-\gamma) \,T_i \alpha.
\end{split}
\end{equation}
Since the triangulated domain is convex, there exist $0 = a_0 < a_1 < \hdots < a_r = 1$ and functions $\alpha_l(\gamma)$ such that for $\gamma \in [a_l,a_{l+1}]$, $0\leq l \leq r-1$ one can write $c(\gamma)=\gamma\, T_j \beta + (1-\gamma)\, T_i \alpha = T_{k_l} {\alpha_l(\gamma)}$ for some $1 \leq k_l \leq T$. The continuity of $c(\gamma)$ implies that $T_{k_l} \alpha_l(a_{l+1}) = T_{k_{l+1}} \alpha_{l+1}(a_{l+1})$, i.e. these points correspond to both simplices, $k_l$ and $k_{l+1}$. Note that this also means that $Q_{k_l}\alpha_l(a_{l+1}) = Q_{k_{l+1}}\alpha_{l+1}(a_{l+1})$. The intuition of this construction is that the $c(a_{l+1})$ are located on the boundaries of adjacent simplices on the line segment. We find 
\begin{equation}
\begin{split}
\|T_i \alpha- T_j \beta\| &= \sum_{l=0}^{r-1} (a_{l+1} - a_l) \|T_i \alpha - T_j \beta\| \\
&=\sum_{l=0}^{r-1} \|(a_{l+1} - a_l)(T_i \alpha - T_j \beta)\| \\
&=\sum_{l=0}^{r-1} \|a_{l+1} T_i \alpha - a_l T_i \alpha - a_{l+1} T_j \beta +  a_l T_j \beta\| \\
&= \sum_{l=0}^{r-1}\left \|  a_l T_j \beta + (1-a_l) T_i\alpha - \left( a_{l+1} T_j \beta + (1-a_{l+1}) T_i\alpha\right) \right \|  \\
&= \sum_{l=0}^{r-1} \left \| T_{k_l}\alpha_l(a_l) - T_{k_l}\alpha_l(a_{l+1}) \right \|  \\
&\stackrel{\eqref{eq:helper}}{\geq} \sum_{l=0}^{r-1} \left \| Q_{k_l}\alpha_l(a_l) - Q_{k_l}\alpha_l(a_{l+1}) \right \|  \\
&\geq  \left \|\sum_{l=0}^{r-1} (Q_{k_l}\alpha_l(a_l) - Q_{k_l}\alpha_l(a_{l+1})) \right \|  \\
&=  \left \|\sum_{l=0}^{r-1} (Q_{k_l}\alpha_l(a_l) - Q_{k_{l+1}}\alpha_{l+1}(a_{l+1})) \right \|  \\
&=  \left \|Q_{k_0}\alpha_0(a_0) - Q_{k_r}\alpha_r(a_r) \right \|  \\
&=  \left \|Q_i \alpha - Q_j \beta \right \|,
\end{split}
\end{equation}
which yields the assertion.

\end{proof}

\begin{proof}[Proof of Proposition 4]
Let $\Delta = \conv \{ t^1, \hdots, t^{n+1} \}$ be given by affinely independent vertices $t^i \in \bbR^n$. 
We show that our lifting approach applied to the label space $\Delta$ solves the
convexified unlifted problem, where the dataterm was replaced by its convex hull on $\Delta$.
Let the matrices $T \in \bbR^{n \times (n+1)}$ and $D \in \bbR^{(n+1) \times n}$ be defined through
\begin{equation}
  T = 
  \begin{pmatrix}
    t^1, & \hdots, & t^{n+1}
  \end{pmatrix},~
  D =
  \begin{pmatrix}
    1 & &\\
     & \ddots &\\
     & & 1\\
     -1 & \hdots & -1
  \end{pmatrix}, ~
  TD = 
  \begin{pmatrix}
    t^1 - t^{n+1}, & \hdots, & t^n - t^{n+1}\\
  \end{pmatrix},
\end{equation}
The transformation $x \mapsto t^{n+1} + TD x$ maps 
$\Delta_e = \conv \{ 0, e^1, \hdots, e^n \} \subset \bbR^n$
to $\Delta$. Now consider the 
following lifted function $\ul : \Omega \to \bbR^{n+1}$ parametrized through $\tilde u : \Omega \to \Delta_e$:
\begin{equation}
  \ul(x) = \begin{pmatrix}\tilde u_1(x), & \hdots,  & \tilde u_n(x), & 1 - \sum_{j=1}^n  \tilde u_j(x) \end{pmatrix}.
\end{equation}
Consider a fixed $x \in \Omega$. Plugging this lifted representation into the biconjugate of the lifted dataterm $\dat$ yields:
\begin{equation}
  \begin{aligned}
    \dat^{**}(\ul) &= \underset{v \in \bbR^{n+1}} \sup ~ \iprod{\ul}{\vl} - \underset{\alpha \in \Delta_{n+1}^U} \sup ~ \iprod{\alpha}{\vl} - \rho(T \alpha)\\
 &= \underset{v \in \bbR^{n+1}} \sup ~ \left\langle\left(\tilde u_1(x), ~\hdots, ~\tilde u_n(x), ~1 - \sum_{j=1}^n \tilde u_j(x)\right),\vl \right\rangle - \\
&\qquad \qquad\underset{\alpha \in \Delta_{n+1}^U} \sup ~ \iprod{\alpha}{\vl} - \rho(T \alpha)\\
    &= \underset{v \in \bbR^{n+1}} \sup ~ \iprod{\tilde u}{D^\top \vl} + \vl_{n+1} -\\ & \qquad \underset{\alpha \in \Delta_{n+1}^U} \sup ~ \left\langle\left(\alpha_1, ~\hdots, ~\alpha_n, ~1 - \sum_{j=1}^n \alpha_j\right),\vl\right\rangle -\\
    &\qquad\qquad \rho\left(\sum_{j=1}^{n} \alpha_j t^j + \left(1 - \sum_{j=1}^{n} \alpha_j\right) t^{n+1}\right)\\
    &= \underset{v \in \bbR^{n+1}} \sup ~ \iprod{\tilde u}{D^\top \vl} + \vl_{n+1} - \underset{\alpha \in \Delta_{n+1}^U} \sup ~ \vl_{n+1} + \iprod{\alpha}{D^\top \vl} - \rho(t^{n+1} + TD \alpha)
  \end{aligned}
\end{equation}
Since $D^\top$ is surjective, we can apply the substitution $\tilde v = D^\top \vl$:
\begin{equation}
  \begin{aligned}
    \hdots &= \underset{\tilde v \in \bbR^{n}} \sup ~ \iprod{\tilde u}{\tilde v} - \underset{\alpha \in \Delta_{n+1}^U} \sup ~ \iprod{\alpha}{\tilde v} - \rho(t^{n+1} + TD \alpha) \\
    &= \underset{\tilde v \in \bbR^{n}} \sup ~ \iprod{\tilde u}{\tilde v} - \underset{w \in \Delta} \sup ~ \iprod{(TD)^{-1}(w - t^{n+1})}{\tilde v} - \rho(w).
  \end{aligned}
\end{equation}
In the last step the substitution $w = t^{n+1} + TD \alpha \Leftrightarrow \alpha = (TD)^{-1} (w - t^{n+1})$ was performed. This can be further simplified to
\begin{equation}
  \begin{aligned}
    \hdots &= \underset{\tilde v \in \bbR^{n}} \sup ~ \iprod{\tilde u}{\tilde v} + \iprod{(TD)^{-1} t^{n+1}}{\tilde v} - (\rho + \delta_{\Delta})^*((TD)^{-T} \tilde v) \\
    &= \underset{\tilde v \in \bbR^{n}} \sup ~ \iprod{\tilde u + (TD)^{-1} t^{n+1}}{\tilde v} - (\rho + \delta_{\Delta})^*((TD)^{-T} \tilde v) \\
    &= \underset{\tilde v \in \bbR^{n}} \sup ~ \iprod{TD \tilde u + t^{n+1}}{(TD)^{-T} \tilde v} - (\rho + \delta_{\Delta})^*((TD)^{-T} \tilde v).
  \end{aligned}
\end{equation}
Since $TD$ is invertible we can perform another substitution 
$v' = (TD)^{-T} \tilde v$.
\begin{equation}
  \begin{aligned}
    \hdots &= \underset{v' \in \bbR^{n}} \sup ~ \iprod{TD \tilde u +
      t^{n+1}}{v'} - (\rho + \delta_{\Delta})^*(v') \\
    &= (\rho + \delta_{\Delta})^{**}(t^{n+1} + TD \tilde u).
  \end{aligned}
\end{equation}
The lifted regularizer is given as:
\begin{equation}
    \boldsymbol{R}(\ul) = \underset{\ql : \Omega \to \bbR^{d \times n+1}} \sup \int_{\Omega} \iprod{\ul}{\Div \ql} - \Psi^*(\ql) ~ \mathrm{d}x
\end{equation}
Using the parametrization by $\tilde u$, this can be equivalently written as
\begin{equation}
  \begin{aligned}
    \underset{\ql(x) \in \mathcal{K}} \sup \int_{\Omega} \sum_{j=1}^{n}\tilde u_j\Div (\ql_j - \ql_{n+1}) + \Div \ql_{n+1} ~ \mathrm{d}x,
  \end{aligned}
  \label{eq:reg_first_subst}
\end{equation}
where the set $\mathcal{K} \subset \bbR^{d \times n+1}$ can be written as
\begin{equation}
  \mathcal{K} = \{ \ql \in \bbR^{d \times n+1} ~|~ \norm{D^\top \ql^\top (TD)^{-1}}_{S^{\infty}} \leq 1\}.
\end{equation}
Note that since $\ql_{n+1} \in C_c^{\infty}(\Omega, \bbR^d)$, the last term $\Div \ql_{n+1}$ in \eqref{eq:reg_first_subst} vanishes by partial integration.
With the substituion $\tilde q(x) = D^\top \ql(x)^\top$ we have
\begin{equation}
  \begin{aligned}
    \underset{\tilde q \in \mathcal{ \tilde K} } \sup \int_{\Omega} \iprod{\tilde u}{\Div \tilde q} ~ \mathrm{d}x,
  \end{aligned}
\end{equation}
with set $\mathcal{\tilde K} \subset \bbR^{d \times n}$:
\begin{equation}
  \mathcal{\tilde K} = \{ q \in \bbR^{d \times n} ~|~ \norm{ q (TD)^{-1} }_{S^{\infty}} \leq 1\}.
\end{equation}
Note that since $\ql_i \in C_c^{\infty}(\Omega, \bbR^d)$, the same holds for the linearly transformed $\tilde q$. 
With another substituion $q'(x) = \tilde q(x) (TD)^{-1}$ we have
\begin{equation}
  \begin{aligned}
    \cdots&=\underset{q' \in \mathcal{K}'} \sup \int_{\Omega} \iprod{\tilde u}{\Div q' TD} ~ \mathrm{d}x\\
    &= \underset{q' \in \mathcal{K}'} \sup \int_{\Omega} \iprod{TD \tilde u}{\Div q'} ~ \mathrm{d}x
  \end{aligned}
\end{equation}
where the set $\mathcal{K}' \subset \bbR^{d \times n+1}$ is given as 
\begin{equation}
  \mathcal{K}' = \{ q \in \bbR^{d \times n} ~|~ \norm{ q }_{S^{\infty}} \leq 1\},
\end{equation}
which is the usual unlifted definition of the total variation $TV(t^{n+1} + TD \tilde u)$. 

This shows that the lifting method solves
\begin{equation}
  \underset{\tilde u : \Omega \to \Delta_e} \min ~ \int_{\Omega} (\rho(x, \cdot) + \delta_{\Delta})^{**}(t^{n+1} + TD \tilde u(x)) \mathrm{d}x + \lambda TV(t^{n+1} + TD \tilde u),
  \label{eq:unlifted_problem_shifted}
\end{equation}
which is equivalent to the original problem but with a convexified data term.
\end{proof}

\section{Additional Experiment: Adaptive Denoising}
In this experiment we jointly estimate the mean $\mu$ and variance
$\sigma$ of an image $I : \Omega \to \bbR$ according to a Gaussian model.
The label space is chosen as $\Gamma = [0, 255] \times [1, 10]$ and 
the dataterm as proposed in \cite{goldluecke-et-al-siam-2013}:
\begin{equation}
  \rho(x, \mu(x), \sigma(x))= \frac{(\mu(x) - I(x))^2}{2\sigma(x)^2} +
  \frac{1}{2}\log(2\pi \sigma(x)^2).
  \label{eq:adaptive_denoising}
\end{equation}
As the projection onto the epigraph of $(\rho
+ \delta_{\Delta})^*$ seems difficult to compute, 
we approximate $\rho$ by a piecewise linear function
using $29 \times 29$ sublabels and convexify it using the quickhull algorithm~\cite{barber1996quickhull}.  
In Fig.~\ref{fig:cartoon-text-decomp} we show the result of minimizing
\eqref{eq:adaptive_denoising} with total variation regularization.
\figCartoonTextDecomp
\end{appendix}

\bibliographystyle{splncs03}
\bibliography{references}
\end{document}

%% file: macros_eccv.tex
\usepackage{pgfplots}
\pgfplotsset{compat=newest}
\usetikzlibrary{plotmarks}
\usepackage{grffile}

\usepackage{latexsym}
\usepackage{babel}
\usepackage{subfig}
\usepackage{multirow}
\usepackage{nicefrac}

\usepackage[bordercolor=white,backgroundcolor=gray!30,linecolor=black,colorinlistoftodos]{todonotes}

\usepackage{wrapfig,graphicx}
\usepackage{url}

\usepackage{color}
\usepackage{colortbl}

\newcommand{\bitem}{\begin{itemize}}
\newcommand{\eitem}{\end{itemize}}

\newcommand{\R}{\mathbb{R}}

\newcommand{\bpm}{\begin{pmatrix}}      
\newcommand{\epm}{\end{pmatrix}}

\newcommand{\la}{\langle}
\newcommand{\ra}{\rangle}

\newcommand{\tmop}[1]{\ensuremath{\operatorname{#1}}}
\newcommand{\norm}[1]{\Vert #1 \Vert}

\newcommand{\bbR}{\mathbb{R}}

\providecommand{\iprod}[2]{\langle#1,#2\rangle}

\newcommand{\bi}{\begin{itemize}}
\newcommand{\ei}{\end{itemize}}

\newcommand{\geqs}{\geqslant}

\newcommand{\cref}[1]{ {\tiny[{#1}]}}
\newcommand{\tm}[1]{}

\newcommand{\Rd}{\mathbb{R}^d}

\newcommand{\beq}{\begin{equation}}
\newcommand{\eeq}{\end{equation}}
\newcommand{\beqa}{\begin{eqnarray}}
\newcommand{\eeqa}{\end{eqnarray}}
\newcommand{\bc}{\begin{center}}
\newcommand{\ec}{\end{center}}

\newcommand{\Om}{\Omega}

\newcommand \TV         {{TV}}                              

\newcommand{\Div}{\tmop{Div}}

\newtheorem{prop}{Proposition}

\newcommand{\dat}{\boldsymbol{\rho}}

\newcommand{\ul}{\boldsymbol{u}}
\newcommand{\vl}{\boldsymbol{v}}

\newcommand{\pl}{\boldsymbol{p}}
\newcommand{\ql}{\boldsymbol{q}}

\newcommand{\conv}{\mathsf{conv}}